\documentclass{article}

\usepackage{times}
\usepackage{graphicx} 
\usepackage{subfigure} 

\usepackage{booktabs}
\usepackage{natbib}

\usepackage{hyperref}
\usepackage{amsfonts}
\usepackage{mathrsfs}
\usepackage{amsmath, amsthm}
\usepackage{amssymb,url,paralist}

\newtheorem{thm}{Theorem}

\newtheorem{lemma}{Lemma}

\newtheorem{definition}[thm]{Definition}
\newtheorem{assumption}{Assumption}

\usepackage[accepted]{icml2015}

\def \R {\mathbb{R}}

\def \y {\mathbf{y}}
\def \x {\mathbf{x}}

\def \xh {\widehat{\x}}

\def \Er {\mathcal{E}}
\def \u {\mathbf{u}}

\def \w {\mathbf{w}}

\def \R {\mathbb{R}}

\def \at {\widetilde{\a}}

\def \xt {\widetilde{\x}}

\def \R {\mathbb{R}}

\def \Se {\mathcal{S}}

\def \x {\mathbf{x}}
\def \xh {\widehat{\mathbf{x}}}
\def \wh {\widehat{\mathbf{w}}}

\def \w {\mathbf{w}}
\def \wt {\widetilde{\mathbf{w}}}

\def \y {\mathbf{y}}
\def \u {\mathbf{u}}

\def \Xh {\widehat X}

\def \ah{\widehat{\alpha}}
\def \at{\widetilde{\alpha}}

\icmltitlerunning{Dual-sparse Regularized Randomized Reduction (DSRR)}

\begin{document} 

\twocolumn[
\icmltitle{Theory of Dual-sparse Regularized Randomized Reduction }

\icmlauthor{Tianbao Yang}{tianbao-yang@uiowa.edu}
\icmladdress{Department of Computer Science, the University of Iowa, Iowa City, USA}
\icmlauthor{Lijun Zhang}{zhanglj@lamda.nju.edu.cn}
\icmladdress{National Key Laboratory for Novel Software Technology, Nanjing University, Nanjing, China}
\icmlauthor{Rong Jin}{rongjin@cse.msu.edu}
\icmladdress{Department of Computer Science and Engineering, Michigan State University, East Lansing, USA\\Institute of Data Science and Technologies at Alibaba Group, Seattle, USA}
\icmlauthor{Shenghuo Zhu}{shenghuo@gmail.com}
\icmladdress{Institute of Data Science and Technologies at Alibaba Group, Seattle, USA}
\icmlkeywords{boring formatting information, machine learning, ICML}

\vskip 0.3in
]

\begin{abstract} 
In this paper, we study  {\it randomized reduction methods}, which reduce high-dimensional features into low-dimensional space by randomized methods (e.g., random projection, random hashing),  for large-scale high-dimensional classification. Previous theoretical results  on randomized reduction methods hinge on strong assumptions about the data, e.g., low rank of the data matrix or  a large separable margin of classification, which hinder their applications  in broad domains. To address these limitations, we propose dual-sparse regularized randomized reduction methods that introduce a sparse regularizer into the reduced dual problem.  Under a mild condition that  the original dual solution is a (nearly) sparse vector, we show that the resulting dual solution  is close to the original dual solution and concentrates on its support set. In numerical experiments, we present an empirical study to support the analysis and  we also present a novel application of the dual-sparse regularized randomized reduction methods to reducing the communication cost of  distributed learning from large-scale high-dimensional data. 
\end{abstract}
\setlength{\belowdisplayskip}{1.5pt} \setlength{\belowdisplayshortskip}{1.2pt}
\setlength{\abovedisplayskip}{1.5pt} \setlength{\abovedisplayshortskip}{1.2pt}
\section{Introduction}  
As the scale and dimensionality of data continue to grow in many applications (e.g., bioinformatics, finance, computer vision, medical informatics)~\cite{DBLP:journals/ijcv/SanchezPMV13,mitchell-2004-learning,Simianer:2012:JFS:2390524.2390527,Arxiv:2011:Portfolio}, it becomes critical to develop efficient and effective algorithms to solve big data machine learning problems. Randomized reduction methods for large-scale or high-dimensional data analytics  have received a great deal of attention in recent years~\cite{mahoney2009matrix,Margin_RP,RP:SVM,DBLP:conf/icml/WeinbergerDLSA09,mahoney-2011-randomized}. 
By either  reducing the dimensionality (\textit{referred to as feature reduction}) or  reducing the number of training instances (\textit{referred to as instance reduction}), the resulting problem has a smaller size of training data that is not only memory-efficient but also computation-efficient.  While randomized instance reduction  has been studied a lot for  fast least square regression~\cite{journals/siammax/DrineasMM08,conf/soda/DrineasMM06,Drineas:2011:FLS:1936922.1936925,DBLP:conf/icml/MaMY14}, randomized feature reduction is more popular for linear classification~\cite{conf/slsfs/Blum05,Margin_RP,RP:SVM,DBLP:conf/icml/WeinbergerDLSA09,Shi:2009:HKS:1577069.1755873} (e.g., random hashing is a noticeable built-in tool in Vowpal Wabbit~\footnote{\url{http://hunch.net/~vw/}}, a fast learning library, for solving high-dimensional problems.). In this paper, we focus on the latter technique and refer to randomized feature reduction as randomized reduction for short. 

Although  several theoretical properties have been examined  for randomized reduction methods when applied to classification, e.g., generalization performance~\cite{RP:SVM}, preservation of margin~\cite{conf/slsfs/Blum05,Balcan:2006:KFK:1164582.1164583,Margin_RP} and the recovery error of the model~\cite{DBLP:journals/tit/0005MJYZ14}, these previous results reply on strong assumptions about the data. For example, both \cite{RP:SVM} and \cite{DBLP:journals/tit/0005MJYZ14} assume the data matrix is of low-rank, and \cite{conf/slsfs/Blum05,Balcan:2006:KFK:1164582.1164583,Margin_RP} make a  assumption that all examples in the original space are separated with a positive margin (with a high probability). Another analysis in~\cite{DBLP:journals/tit/0005MJYZ14} assumes the weight vector for classification is sparse.  These assumptions are too strong to hold in many real applications. 

\textbf{Contributions.} To address these limitations,  we propose dual-sparse regularized randomized reduction methods referred to as {\bf DSRR} by leveraging the (near) sparsity of dual solutions for large-scale high-dimensional ({\bf LSHD}) classification problems (i.e., the number of (effective) support vectors is small compared to the total number of examples). In particular, we add a dual-sparse regularizer into the reduced dual problem. We present a novel theoretical analysis of the recovery error of the dual variables and the primal variable and study its implication for different randomized reduction methods (e.g., random projection, random hashing and random sampling). 

\textbf{Novelties.} Compared with previous works~\cite{conf/slsfs/Blum05,Balcan:2006:KFK:1164582.1164583,Margin_RP,RP:SVM}, our theoretical analysis demands a mild assumption about the data and directly provides guarantee on  a small recovery error of the obtained model, which is critical 
 for subsequent analysis, e.g., feature selection~\cite{Guyon:2002:GSC:599613.599671,Brank02featureselection} and model interpretation~\cite{citeulike:399676,conf/icml/SonnenburgF10,conf/ismb/RatschSS05,journals/bmcbi/SonnenburgSPBR07,BOSSR2008}. For example, when exploiting a linear model to classify people into sick or not sick based on genomic markers, the learned weight vector is important for understanding the effect of different  genomic markers on the disease and for designing effective medicine~\cite{Jostins15102011,Kang11}. In addition, the  recovery  could also increase the predictive performance, in particular when there exists noise in the original features~\cite{Goldberger2005}.

Compared with~\cite{DBLP:journals/tit/0005MJYZ14} that proposes to recover a  linear model in the original feature space by dual recovery, i.e., constructing a weight vector using the dual variables learned from the reduced problem and the original feature vectors, our methods are better in that  (i) we rely on a more realistic assumption of the sparsity of dual variables (e.g., in support vector machine (SVM)); (ii) we analyze both smooth loss functions and  non-smooth loss functions (they focused on smooth functions); (iii) we study different randomized reduction methods in the same framework not just the random projection. 

In numerical experiments, we present an empirical study on a real data set to support our analysis and we also demonstrate a novel application of the reduction and recovery framework in distributed  learning from LSHD data, which combines the benefits of the two complementary techniques for addressing big data problems. Distributed learning/optimization recently receives significant interest in  solving big data problems~\cite{DBLP:conf/nips/JaggiSTTKHJ14,NIPS2014_5597,DBLP:conf/nips/Yang13,journals/corr/abs-1110-4198}. However, it is notorious for high communication cost, especially when the dimensionality of data is very high. By solving a dimensionality reduced data problem and using the recovered solution as an initial solution to the distributed optimization on the original data, we can reduce the number of iterations and the communication cost. In practice, we employ the recently developed distributed stochastic dual coordinate ascent  algorithm~\cite{DBLP:conf/nips/Yang13}, and observe that using the recovered solution as an initial solution we are able to attain almost the same performance  with only one or two communications of  high dimensional vectors among multiple machines. 

\section{Preliminaries}
Let $(\x_i, y_i), i=1,\ldots, n$ denote a set of training examples, where $\x_i\in\R^d, y_i\in\{1,-1\}$. Assume both $n$ and $d$ are very large.  The goal of classification is to solve the following optimization problem: 
\begin{align}\label{eqn:primal}
\w_* = \arg\min_{\w\in\R^d} \frac{1}{n}\sum_{i=1}^n\ell(\w^{\top}\x_iy_i) + \frac{\lambda}{2}\|\w\|^2_{2}
\end{align}
where $\ell(zy)$ is a convex loss function and $\lambda$ is a regularization parameter. 
Using the conjugate function, we can turn the problem into a dual problem:
\begin{align}\label{eqn:dual}
\alpha_* = \arg\max_{\alpha\in\R^n} -\frac{1}{n}\sum_{i=1}^n\ell_i^*(\alpha_i) - \frac{1}{2\lambda n^2}\alpha^TX^{\top}X\alpha
\end{align}
where $X=(\x_1,\ldots, \x_n)$ is the data  matrix and $\ell^*_i(\alpha)$ is the convex conjugate function of $\ell(zy_i)$. Given the optimal dual solution $\alpha_*$, the optimal primal solution can be computed by 
$\w_* = -\frac{1}{\lambda n}X\alpha_*$. 
For LSHD problems, directly solving the primal problem~(\ref{eqn:primal}) or the dual problem~(\ref{eqn:dual}) could be very expensive. We aim to address the challenge by randomized reduction methods. Let $A(\cdot):\R^d\rightarrow \R^m$ denote a randomized reduction operator that reduces a $d$-dimensional feature vector into $m$-dimensional feature vector. Let $\xh = A(\x)$ denote the reduced feature vector. With the reduced feature vectors $\xh_1,\ldots, \xh_n$ of the training examples,  a conventional approach is to solve the following reduced primal problem
\begin{align}\label{eqn:rprimal}
\u_* = \arg\min_{\u\in\R^m} \frac{1}{n}\sum_{i=1}^n\ell(\u^{\top}\xh_iy_i) + \frac{\lambda}{2}\|\u\|^2_{2}
\end{align}
or its the dual problem 
\begin{align}\label{eqn:rdual}
\ah_* = \arg\max_{\alpha\in\R^n} -\frac{1}{n}\sum_{i=1}^n\ell_i^*(\alpha_i) - \frac{1}{2\lambda n^2}\alpha^T\Xh^{\top}\Xh\alpha
\end{align}
where $\Xh = (\xh_1,\ldots, \xh_n)\in\R^{m\times n}$.  Previous studies have analyzed the reduced problems for random projection methods and proved the preservation of margin~\cite{conf/slsfs/Blum05,Margin_RP} and the preservation of minimum enclosing ball~\cite{RP:SVM}. \citet{DBLP:journals/tit/0005MJYZ14} proposed a dual recovery approach that constructs a recovered solution by  $\wh_* = -\frac{1}{\lambda n}\sum_{i=1}^n[\ah_*]_i\x_i$ and proved the recovery error for random projection under the assumption of  low-rank data matrix or sparse  $\w_*$. In addition, they also showed that the naive recovery by $A^{\top}\u_*$ (when $A(\x)=A\x$) has a large recovery error. 

One deficiency with the simple dual recovery approach is that due to the reduction in the feature space, many non-support vectors for the original optimization problem will become support vectors, which could result in the corruption in the recovery error. As a result, the original analysis of dual recovery method requires a strong assumption of data (i.e., the low rank assumption). In this work, we plan to address this limitation in a different way, which allows us to relax the assumption significantly. 

\section{DSRR and its Guarantee}
To reduce the number of or the contribution of training instances that are non-support vectors in the original optimization problem and are transformed into support vectors due to the reduction of the feature space,  we employ a simple trick that adds a dual-sparse regularization  to the reduced dual problem. In particular, we solve the following problem:
\begin{align}\label{eqn:sdual}
&\at_* = \\
&\arg\max_{\alpha\in\R^n} -\frac{1}{n}\sum_{i=1}^n\ell_i^*(\alpha_i) - \frac{1}{2\lambda n^2}\alpha^T\Xh^{\top}\Xh\alpha  - \frac{1}{n}R(\alpha)\nonumber
\end{align}
where $R(\alpha) = \tau\|\alpha\|_1$, 
and $\tau>0$ is a regularization parameter, whose theoretical value will be revealed later. 


 To further understand the added dual-sparse regularizer,  we consider SVM, where the loss function can be either the hinge loss (a non-smooth function) $\ell(zy)=\max(0, 1-zy)$ or the squared hinge loss (a smooth function) $\ell(zy) = \max(0, 1 -zy)^2$. We first consider the hinge loss, where $\ell_i^*(\alpha_i) = \alpha_i y_i$ for $\alpha_iy_i\in[-1, 0]$.  Then the new dual problem is equivalent to 
\begin{align*}
\max_{\alpha\circ \y\in[-1, 0]^n} \frac{1}{n}\sum_{i=1}^n-\alpha_iy_i - \frac{1}{2\lambda n^2}\alpha^T\Xh^{\top}\Xh\alpha - \frac{\tau}{n}\|\alpha\|_1
\end{align*}
Using variable transformation $-\alpha_i y_i \rightarrow \beta_i$, the above problem is equivalent to  
\begin{align*}
\max_{\beta\in[0,1]^n} \frac{1}{n}\sum_{i=1}^n\beta_i (1-\tau)- \frac{1}{2\lambda n^2}(\beta\circ \y)^T\Xh^{\top}\Xh(\beta\circ\y) 
\end{align*}
Changing into the primal form, we have  
\begin{align}\label{eqn:dual-app-33}
\max_{\u\in\R^m} \frac{1}{n}\sum_{i=1}^n\ell_{1-\tau}(\u^{\top}\xh_iy_i) + \frac{\lambda}{2}\|\u\|_2^2
\end{align}
where $\ell_\gamma(z)  = \max(0, \gamma - z)$ is a max-margin loss with margin given by $\gamma$. It can be understood that adding the $\ell_1$ regularization in the reduced problem of SVM is equivalent to using a max-margin loss with a smaller margin, which is intuitive because examples become difficult to separate after dimensionality reduction and is consistent with several previous studies that the margin is reduced in the reduced feature space~\cite{conf/slsfs/Blum05,Margin_RP}. Similarly for squared hinge loss, the equivalent  primal problem  is
\begin{align}\label{eqn:dual-app-34}
\max_{\u\in\R^m} \frac{1}{n}\sum_{i=1}^n\ell^2_{1-\tau}(\u^{\top}\xh_iy_i) + \frac{\lambda}{2}\|\u\|_2^2
\end{align}
where $\ell^2_\gamma(z)  = \max(0, \gamma - z)^2$. 

Although adding a dual-sparse regularizer is intuitive and can be motivated from previous results,  we emphasize that the proposed dual-sparse formulation provides a new perspective and bounding the dual recovery error  $\|\at_*-\alpha_*\|$ is a non-trivial task, which is a major contribution of this paper. We first state our main result in Theorem~\ref{thm:1} for  smooth loss functions. 

\begin{thm}\label{thm:1}
Let $\at_*$ be the optimal dual solution to~(\ref{eqn:sdual}). Assume $\alpha_*$ is $s$-sparse with the support set given by $\Se$. If $\tau\geq \frac{2}{\lambda n}\|(X^{\top}X - \Xh^{\top}\Xh)\alpha_*\|_\infty$, then we have
\begin{align}\label{eqn:con}
\|[\at_*]_{\Se^c}\|_1&\leq 3 \|[\at_*]_{\Se}-[\alpha_*]_{\Se}\|_1
\end{align}
Furthermore, if $\ell(z)$ is a $L$-smooth loss function~\footnote{A function is $L$-smooth if its gradient is $L$-Lipschitz continuous. }, we have
\begin{align}
& \|\at_* -\alpha_*\|_2\leq 3\tau L \sqrt{s},\quad \|\at_* -\alpha_*\|_1\leq 12\tau L s \\
&\|[\at_*]_{\Se} - [\alpha_*]_{\Se}\|_1\leq 3\tau Ls, \quad \|[\at_*]_{\Se^c}\|_1\leq 9\tau Ls\label{eqn:conc}
\end{align}
where $\Se^c$ is the complement of $\Se$, and $[\alpha]_\Se$ is a vector that only contains the elements of $\alpha$ in the set $\Se$.
\end{thm} 
{\bf Remark 1:} 
The proof is presented in Appendix~\ref{sec:A}. 
It can be seen that  the dual recovery error is proportional to the value of $\tau$ which is dependent on $\|(X^{\top}X - \Xh^{\top}\Xh)\alpha_*\|_\infty$, {\it which we can bound  without using any assumption about the data matrix or the optimal dual variable $\alpha_*$}. In contrast, previous bounds~\cite{COLT13:Zhang,DBLP:journals/tit/0005MJYZ14,RP:SVM} depend on $\|X^{\top}X - \Xh^{\top}\Xh\|_2$, which requires the low rank assumption on $X$.  In next section, we provide an upper bound of $ \frac{1}{\lambda n}\|(X^{\top}X - \Xh^{\top}\Xh)\alpha_*\|_\infty$ that will allow us to understand how the reduced dimensionality $m$ affects the recovery error. Essentially, the results indicate that for random projection, randomized Hadamard transform and random hashing, $ \frac{1}{\lambda n}\|(X^{\top}X - \Xh^{\top}\Xh)\alpha_*\|_\infty\leq O(\sqrt{\frac{\log(n/\delta)}{m}})\|\w_*\|_2$ with a high probability $1-\delta$, and thus the recovery error will be scaled as $\sqrt{1/m}$ in terms of $m$ - the same order of recovery error as in~\cite{COLT13:Zhang,DBLP:journals/tit/0005MJYZ14} that assumes low rank of the data matrix. 

{\bf Remark 2:} We would like to make a connection with LASSO for sparse signal recovery.  In sparse signal recovery under noise measurements $\mathbf f = U\w_*+\mathbf e$, where $\mathbf e$ denotes the noise in measurements, if a LASSO $\min_{\w}\frac{1}{2}\|U\mathbf w - \mathbf f\|_2^2 + \lambda \|\w\|_1$ is solved for the solution, then the regularization parameter $\lambda$ is required to be
larger than the quantity $\|U^{\top}\mathbf e\|_\infty$ that depends on the noise in order to have an accurate recovery~\cite{eldar2012compressed}.  Similarly in our formulation, the added $\ell_1$ regularization $\tau\|\alpha\|_1$  is to counteract  the noise in  $\Xh\Xh^{\top}$ as compared with $XX^{\top}$ and the value of $\tau$ is dependent on the noise. 

To present the theoretical result on the non-smooth loss functions, we need to introduce restricted eigen-value conditions similar to those used in the sparse recovery analysis for LASSO~\cite{Bickel09simultaneousanalysis,DBLP:journals/siamjo/Xiao013}
. In particular, we introduce the following definition of restricted eigen-value condition. 
\begin{definition}
Given an integer $s>0$, we define 
\begin{align*}
\mathcal K_{n,s}=\{\alpha\in\R^n: \|\alpha\|_2\leq  1,  \|\alpha\|_1\leq \sqrt{s}\}.
\end{align*}
We say that $X$ satisfies the restricted eigenvalue condition at sparsity level s if there exist positive constants $\rho^+_s$ and $\rho^{-}_s$ such that
\begin{align*}
\rho^+_{s} = \sup_{\alpha\in\mathcal K_{n, s}}\frac{\alpha^{\top}X^{\top}X\alpha}{n}, \quad \rho^{-}_{s} = \inf_{\alpha\in\mathcal K_{n,s}} \frac{\alpha^{\top}X^{\top}X\alpha}{n}.
\end{align*}
\end{definition}
\vspace*{-0.1in}
We also define another quantity that measures the restricted eigen-value of $X^{\top}X - \Xh^{\top}\Xh$, namely 
\begin{align}
\sigma_s = \sup_{\alpha\in \mathcal K_{n,s}}\frac{|\alpha^{\top}(X^{\top}X - \Xh^{\top}\Xh)\alpha|}{n}. 
\end{align}
\begin{thm}\label{thm:2}
Let $\at_*$ be the optimal dual solution to~(\ref{eqn:sdual}). Assume $\alpha_*$ is $s$-sparse with the support set given by $\Se$. If $\tau\geq \frac{2}{\lambda n}\|(X^{\top}X - \Xh^{\top}\Xh)\alpha_*\|_\infty$, then we have
\begin{align*}
\|[\at_*]_{\Se^c}\|_1&\leq 3 \|[\at_*]_{\Se}-[\alpha_*]_{\Se}\|_1
\end{align*}
Assume the data matrix $X$ satisfies the restricted eigen-value condition at sparsity level $16s$ and $\sigma_{16s}<\rho^-_{16s}$, we have
\begin{align*}
 \|\at_* -\alpha_*\|_2&\leq \frac{3\lambda}{2(\rho^-_{16s} - \sigma_{16s})}\tau \sqrt{s}\\
 \|\at_* -\alpha_*\|_1&\leq \frac{6\lambda}{(\rho^-_{16s} - \sigma_{16s})}\tau s
\end{align*}
\end{thm}
{\bf Remark 3:} The proof is included in Appendix~\ref{sec:B}. Compared to smooth loss functions, the conditions that guarantee a small recovery for non-smooth loss functions  are more restricted. In next section, we will provide a bound on $\sigma_{16s}$ to further understand the condition of $\sigma_{16s}\leq \rho^{-}_{16s}$, which essentially implies that $m\geq\Omega\left(\left(\frac{\rho^+_{16s}}{\rho^-_{16s}}\right)^2s\log(n/s)\right)$. 

Last but not least, we provide a theoretical result on the recovery error for the nearly sparse optimal dual variable $\alpha_*$. We state the result for  smooth loss functions.  To quantify the near sparsity, we let $\alpha_*^s\in\R^n$ denote a vector  that zeros all entries in $\alpha_*$ except for the top-$s$ elements in magnitude and assume $\alpha_*^s$ satisfies the following condition:
\begin{align}\label{eqn:nearly}
&\left\|\nabla\ell^*(\alpha^s_*) + \frac{1}{\lambda n}X^{\top}X\alpha^s_*\right\|_\infty\leq \xi
\end{align}
where $\nabla\ell^*(\alpha) = (\nabla\ell_1^*(\alpha_1),\ldots, \nabla\ell_n^*(\alpha_n))^{\top}$. The above condition can be considered as a sub-optimality condition~\cite{boyd-2004-convex} of $\alpha^s_*$ measured in the infinite norm. For the optimal solution $\alpha_*$, we have $\nabla\ell^*(\alpha_*) + \frac{1}{\lambda n}X^{\top}X\alpha_*=0$. 
\begin{thm}\label{thm:4}
Let $\at_*$ be the optimal dual solution to~(\ref{eqn:sdual}). Assume $\alpha_*$ is nearly $s$-sparse such that~(\ref{eqn:nearly}) holds with the support set of $\alpha^s_*$ given by $\Se$. If $\tau\geq \frac{2}{\lambda n}\|(X^{\top}X - \Xh^{\top}\Xh)\alpha_*\|_\infty + 2\xi$, then we have
\begin{align*}
\|[\at_*]_{\Se^c}\|_1&\leq 3 \|[\at_*]_{\Se}-[\alpha_*]_{\Se}\|_1
\end{align*}
Furthermore, if $\ell(z)$ is a $L$-smooth loss function, we have
\begin{align}
& \|\at_* -\alpha^s_*\|_2\leq 3\tau L \sqrt{s},\quad \|\at_* -\alpha^s_*\|_1\leq 12\tau L s \\
&\|[\at_*]_{\Se} - [\alpha_*]_{\Se}\|_1\leq 3\tau Ls, \quad \|[\at_*]_{\Se^c}\|_1\leq 9\tau Ls\label{eqn:conc}
\end{align}
\end{thm} 
{\bf Remark 4:} The proof appears in Appendix~\ref{sec:C}. Compared to Theorem~\ref{thm:1} for  exactly sparse optimal dual solution, the dual recovery error bound for nearly sparse optimal dual solution is increased by $6L\sqrt{s}\xi$ for $\ell_2$ norm and by $24Ls\xi$ for $\ell_1$ norm. 

Finally, we note that with the recovery error bound for the dual solution, we can easily derive an error bound for the primal solution $\wt_*=-\frac{1}{\lambda n}X\at_*$. Below we present a theorem for smooth loss functions. One can easily extend the result to non-smooth loss functions.
\begin{thm}
Let $\wt_*$ be the recovered primal solution using $\at_*$ the optimal dual solution to~(\ref{eqn:sdual}). Assume $\alpha_*$ is $s$-sparse and $\ell(z)$ is a $L$-smooth loss function. If $\tau\geq \frac{2}{\lambda n}\|(X^{\top}X - \Xh^{\top}\Xh)\alpha_*\|_\infty$ then we have
\begin{align*}
\|\wt_* - \w_*\|_2\leq  \frac{\sigma_1}{\lambda n}3L\tau\sqrt{s}
\end{align*}
where $\sigma_1$ is the maximum singular value of $X$. Furthermore if $\frac{1}{n}X^{\top}X$ has a  restricted eigen-value $\rho^+_{16s}$  at sparsity level $16s$, then 
\begin{align*}
\|\wt_* - \w_*\|_2\leq  \frac{\sqrt{\rho^+_{16s}}}{\lambda\sqrt{n}}3L\tau\sqrt{s}
\end{align*}
 \end{thm}

{\bf Remark 5: } Since $\rho^+_{16s}$ is always less than $\sigma_1^2/n$, the second result if the restricted eigen-value condition holds is always better than the first result. With the bound of $\tau$ as revealed later, we can see that the error of $\wt_*$ scales as $O(\sqrt{\frac{s}{m}}\|\w_*\|_2)$ in terms of sparsity $s$ of $\alpha_*$, the reduced dimensionality  $m$ and the magnitude of $\w_*$. A similar order of error bound was established in~\cite{DBLP:journals/tit/0005MJYZ14} assuming $\w_*$ is $s$-sparse and $X$ is approximately low rank. In contrast, we do not assume $X$ is approximately low rank. 
 

\section{Analysis}\label{sec:ana}
In this section, we first provide upper bound analysis of $ \frac{2}{\lambda n}\|(X^{\top}X - \Xh^{\top}\Xh)\alpha_*\|_\infty$ and $\sigma_{s}$. 
To facilitate our analysis, we define 
\begin{align*}
\Delta = \frac{1}{\lambda n}(\Xh^{\top}\Xh-X^{\top}X)\alpha_*
\end{align*}

\subsection{Bounding $\|\Delta\|_\infty$}
A critical condition in both Theorem~\ref{thm:1} and Theorem~\ref{thm:2} is $\tau> \|\Delta\|_\infty$. In order to reveal the theoretical value of $\tau$ and its implication for various randomized reduction methods, we need to bound $\|\Delta\|_\infty$. We first provide a general analysis and then study its implication for various randomized reduction methods separately.  The analysis is based on the following assumption, which essentially is indicated by Johnson-Lindenstrauss (JL)-type lemmas. 
\begin{assumption}[A1]
Let $A(\x)=A\x$ be a linear projection operator where $A\in\R^{m\times d}$ such that for any given $\x\in\R^d$ with a high probability $1-\delta$, we have
\[
\left|\|A\x\|^2_2-\|\x\|_2^2\right|\leq\epsilon_{A,\delta}\|\x\|_2^2
\]
where $\epsilon_{A,\delta}$ depends on $m$, $\delta$ and possibly $d$. 
\end{assumption}
With this assumption, we have the following theorem regarding the upper bound of $\|\Delta\|_\infty$. 
\begin{thm}
Suppose $A\in\R^{m\times d}$ satisfies Assumption {\bf A}, then with a high probability $1-2\delta$ we have
\begin{align*}
\|\Delta\|_\infty\leq R\|\w_*\|_2 \epsilon_{A,\delta/n}
\end{align*}
where $R=\max_i\|\x_i\|_2$.
\end{thm}
\vspace*{-0.1in}\textit{Proof.}
\begin{align*}
&\frac{1}{\lambda n}( \Xh^{\top}\Xh - X^{\top}X )\alpha_* = \frac{1}{\lambda n}( X^{\top}A^{\top}AX - X^{\top}X)\alpha_* \\
&= \frac{1}{\lambda n}X^{\top}(A^{\top}A - I)X\alpha_* = X^{\top}(I - A^{\top}A)\w_*
\end{align*}
where we use the fact $\w_* = -\frac{1}{\lambda n}X\alpha_*$. Then
\begin{align*}
\frac{1}{\lambda n}[( \Xh^{\top}\Xh - X^{\top}X )\alpha_*]_i = \x_i^{\top}(I - A^{\top}A)\w_*
\end{align*}
Therefore in order to bound $\|\Delta\|_\infty$, we need to bound $\x_i^{\top}(I - A^{\top}A)\w_*$ 
for all $i\in[n]$. We first bound for individual $i$ and then apply the union bound. Let $\xt_i$ and $\wt_*$ be normalized version of $\x_i$ and $\w_*$, i.e., $\xt_i = \x_i/\|\x_i\|_2$ and $\wt_* = \w_*/\|\w_*\|_2$. Suppose Assumption {\bf A} is satisfied, then with a probability $1-\delta$, 
\begin{align*}
 \xt_i^{\top}A^{\top}A\wt_* - \xt_i^{\top}\wt_* &= \frac{\|A(\xt_i+\wt_*)\|_2^2 - \|A(\xt_i-\wt_*)\|_2^2}{4}\\
  &\hspace*{-0.3in} - \xt_i^{\top}\wt_*\leq \frac{\epsilon_{A,\delta}}{2}(\|\xt_i\|_2^2 + \|\wt_*\|_2^2)\leq \epsilon_{A,\delta}
\end{align*}
Similarly with a probability $1- \delta$, 
\begin{align*}
 \xt_i^{\top}A^{\top}A\wt - \xt_i^{\top}\wt_* &= \frac{\|A(\xt_i+\wt_*)\|_2^2 - \|A(\xt_i-\wt_*)\|_2^2}{4}\\
 &\hspace*{-0.5in} - \xt_i^{\top}\wt_*\geq -\frac{\epsilon_{A,\delta}}{2}(\|\xt_i\|_2^2 + \|\wt_*\|_2^2)\geq -\epsilon_{A,\delta}
\end{align*}
Therefore with a probability $1 - 2\delta$, we have
\begin{align*}
 &|\x_i^{\top}A^{\top}A\w_* - \x_i^{\top}\w_*|\\
 &\leq \|\x_i\|_2\|\w_*\|_2 |\xt_i^{\top}A^{\top}A\wt_* - \xt^{\top}\wt_*|\leq \|\x_i\|_2\|\w_*\|_2 \epsilon_{A,\delta}
\end{align*}
Then applying union bound, we complete the proof. 

Next, we discuss four classes of randomized reduction operators, namely random projection, randomized Hadamard transform, random hashing and random sampling, and study the corresponding $\epsilon_{A,\delta}$ and their implications for the recovery error. 

\textbf{Random Projection.} Random projection has been employed widely for dimension reduction. The projection operator $A$ is usually sampled from sub-Gaussian distributions with mean $0$ and variance $1/m$, e.g., (i) Gaussian distribution: $A_{ij}\sim\mathcal N(0, 1/m)$, (ii) Rademacher distribution: $\Pr(A_{ij} =\pm 1/\sqrt{m}) =0.5$, (iii) discrete distribution: $\Pr(A_{ij}=\pm\sqrt{3/m})=1/6$ and $\Pr(A_{ij}=0)=2/3$. The last two distributions for dimensionality reduction were proposed and analyzed  in~\cite{Achlioptas:2003:DRP:861182.861189}. 
The following lemma is the general JL-type lemma for $A$  with sub-Gaussian entries, which reveals the value of $\epsilon_{A,\delta}$ in Assumption {\bf A}. 
\begin{lemma}\cite{nelson2010}\label{lem:JL}
Let $A\in\R^{m\times d}$ be a random matrix with subGaussian entries of mean $0$ and variance $1/m$ . For any given $\x$ with a probability $1-\delta$, we have
\[
  \left|\|A\x\|^2_2 - \|\x\|^2_2\right| \leq c\sqrt{\frac{\log(1/\delta)}{m}}\|\x\|^2_2
\]
where $c$ is some small universal constant. 
\end{lemma} 
%

\textbf{Randomized Hadamard Transform.} Randomized Hadamard transform was introduced to speed-up random projection, reducing the computational time~\footnote{refers to the running time of computing $A\x$. } of random projection from $O(dm)$ to $O(d\log d)$ or even $O(d\log m)$. The projection matrix $A$ is of the form $A=PHD$, where
\begin{itemize}
\vspace*{-0.1in}
\item $D\in\R^{d\times d}$ is a diagonal matrix with $D_{ii} = \pm1$ with equal probabilities. 
\vspace*{-0.1in}
\item $H$ is the $d\times d$ Hadamard matrix (assuming $d$ is a power of $2$), scaled by $1/\sqrt{d}$.
\vspace*{-0.05in}
\item $P\in\R^{m\times d}$ is typically a sparse matrix that facilities computing $P\x$. Several choices of $P$ are possible~\cite{nelson2010,ailon-2009-fast,DBLP:journals/aada/Tropp11}.  Below we provide a JL-type lemma for a randomized Hadamard transform with $P\in\R^{m\times d}$ that samples $m$ coordinates from $\sqrt{\frac{d}{m}}HD\x$ with replacement.
\end{itemize}
\vspace*{-0.1in}
\begin{lemma}\cite{nelson2010}\label{lem:2}
Let $A=\sqrt{\frac{d}{m}}PHD\in\R^{m\times d}$ be a randomized Hadamard transform with $P$ being a random sampling matrix.  For any given $\x$ with a probability $1-\delta$, we have
\[
  \left|\|A\x\|^2_2 - \|\x\|^2_2\right| \leq c\sqrt{\frac{\log(1/\delta)\log(d/\delta)}{m}}\|\x\|^2_2
\]
where $c$ is some small universal constant. 
\end{lemma} 
{\bf Remark 6:} Compared to random projection, there is an additional $\sqrt{\log(d/\delta)}$ factor in $\epsilon_{A,\delta}$. However, it can be  removed by applying an additional random projection. In particular, if we let $A = \sqrt{\frac{d}{m}}P'PHD\in\R^{m\times d}$, where $P\in\R^{t\times d}$ is a random sampling matrix with $t = m\log(d/\delta)$ and $P'\in\R^{m\times t}$ is a random projection matrix that satisfies Lemma~\ref{lem:JL}, then we have the same order of $\epsilon_{A,\delta}$. Please refer to~\cite{nelson2010} for more details. 

\textbf{Random Hashing.} 
Another line of work to speed-up random projection is random hashing which makes the projection matrix $A$ much sparser and takes advantage of the sparsity of feature vectors. It was introduced  in~\cite{DBLP:journals/jmlr/ShiPDLSSV09} for dimensionality reduction and later was improved to an unbiased version by~\cite{DBLP:conf/icml/WeinbergerDLSA09} with some theoretical analysis. \citet{Dasgupta:2010:SJL} provided a rigorous analysis of the unbiased random hashing. Recently, \citet{Kane:2014:SJT:2578041.2559902} proposed two new random hashing algorithms with a slightly sparser random matrix $A$. Here we provide a JL-type lemma for the random hashing algorithm in~\cite{DBLP:conf/icml/WeinbergerDLSA09,Dasgupta:2010:SJL}.  Let $h:\mathbb N\rightarrow [m]$ denote a random hashing function, and $\xi=(\xi_1,\ldots, \xi_d)$ denote a Rademacher random variable, i.e., $\xi_i, i=1, \ldots, d$ are independent and $\xi_i\in\{1,-1\}$ with equal probabilities. The projection matrix $A$ can be written as $A=HD$, where $D\in\R^{d\times d}$ is a diagonal matrix with $D_{jj} = \xi_j$, and  $H\in\R^{m\times d}$ with $H_{ij} =\delta_{i,h(j)}$~\footnote{$\delta_{ij}=1$ if $i=j$, and $0$ otherwise.}.  Under the random matrix $A$, the feature vector $\x\in\R^d$ is reduced to $\xh\in\R^m$, where $[\xh]_i = \sum_{j: h(j)=i}[\x]_j\xi_j$. The following JL-type Lemma is a basic result from~\cite{Dasgupta:2010:SJL} with a rephrasing. 
\begin{lemma}
Let $A=HD\in\R^{m\times d}$ be a random hashing matrix.  For any given vector $\x\in\R^d$ such that $\frac{\|\x\|_\infty}{\|\x\|_2}\leq \frac{1}{\sqrt{c}}$, for $\delta<0.1$, with a probability $1-3\delta$, we have
\begin{align*}
\left|\|A\x\|_2^2 - \|\x\|_2^2\right|\leq \sqrt{\frac{12\log(1/\delta)}{m}}\|\x\|_2^2
\end{align*}
where $c = 8\sqrt{m/3}\log^{1/2}(1/\delta)\log^2(m/\delta)$.
\end{lemma}
{\bf Remark 7:} Compared to random projection, there is an additional condition on the feature vector $\|\x\|_\infty\leq \frac{\|\x\|_2}{\sqrt{c}}$. However,  it can be removed by applying an extra preconditioner $P$ to $\x$ before applying the projection matrix $A$, i.e., $\xh = HDP\x$. Two preconditioners were discussed in~\cite{Dasgupta:2010:SJL}, with one corresponding to duplicating $\x$ $c$ times and scaling it by $1/\sqrt{c}$ 
 and another one given by $P\in\R^{d\times d}$ which consists of $d/b$ diagonal blocks of $b\times b$ randomized Hadamard matrix, where $b = 6c\log(3c/\delta)$. The running time of the reduction using the later preconditioner is $O(d\log c\log\log c)$.
\vspace*{-0.15in}
\paragraph{Random Sampling. }
Last we discuss random sampling and compare with the aforementioned randomized reduction methods. In fact, the JL-type lemma for random sampling  is implicit in the proof of Lemma~\ref{lem:2}. We make it explicit in the following lemma. 
\begin{lemma}
Let $A=\sqrt{\frac{d}{m}}P\in\R^{m\times d}$ be a scaled random sampling matrix where $P\in\R^{m\times d}$ samples $m$ coordinates with replacement. Then with a probability $1-\delta$, we have
\[
\left|\|A\x\|_2^2 - \|\x\|_2^2\right|\leq \frac{\|\x\|_\infty}{\|\x\|_2}\sqrt{\frac{3d\log(1/\delta)}{m}}\|\x\|_2^2
\]
\end{lemma}
\vspace*{-0.1in}
{\bf Remark 8: } Compared with other three randomized reduction methods, there is an additional $\frac{\|\x\|_\infty}{\|\x\|_2}\sqrt{d}$ factor in $\epsilon_{A,\delta}$, which could result in a much larger $\epsilon_{A,\delta}$ and consequentially a larger recovery error. That is why the randomized Hadamard transform was introduced to make this additional factor close to a constant.

From the above discussions, we can conclude  that with random projection, randomized Hadamard transform and random hashing,  with a probability $1-\delta$ we have,
\begin{align*}
\|\Delta\|_\infty&=\max_i|\x_i^{\top}(I - A^{\top}A)\w_*|\\
&\leq cR\sqrt{\frac{\log(n/\delta)}{m}}\|\w_*\|_2.
\end{align*}
which essentially indicates that $\tau\geq 2cR\sqrt{\frac{\log(n/\delta)}{m}}\|\w_*\|_2$. 

\subsection{Bounding $\sigma_s$ for non-smooth case}
Another condition in Theorem~\ref{thm:2} is to require $\sigma_{16s}\leq \rho^-_{16s}$. Since $\rho^-_{16s}$ is dependent on the data, we provide an upper bound of $\sigma_{16s}$ to further understand the condition. In the following analysis, we assume $\epsilon_{A,\delta}=O(\sqrt{\frac{\log(1/\delta)}{m}})$. Recall the definition of $\sigma_s$:
\begin{align}
\sigma_s = \sup_{\alpha\in \mathcal K_{n,s}}\frac{|\alpha^{\top}(X^{\top}X - \Xh^{\top}\Xh)\alpha|}{n}. 
\end{align}
We provide a bound of $\sigma_s$ below. 

The key idea is to use the convex relaxation of $\mathcal K_{n,s}$. Define $\mathcal S_{n,s} = \{\alpha\in\R^n: \|\alpha\|_2\leq  1,  \|\alpha\|_0\leq s \}$. It was shown in~\cite{DBLP:journals/corr/abs-1109-4299} that $conv(\mathcal S_{n,s})\subset \mathcal K_{n,s}\subset 2conv(\mathcal S_{n,s})$, where $conv(\mathcal S)$ is the convex hull of the set $\mathcal S$. 
It is not difficult to show that (see the supplement)
\begin{align*}
&\max_{\alpha\in\mathcal K_{n,s}}|(X\alpha)^{\top}(I - A^{\top}A)(X\alpha)|\\
&\leq 4\max_{\alpha_1, \alpha_2\in \mathcal S_{n,s}}|(X\alpha_1)^{\top}(I - A^{\top}A)(X\alpha_2)|
\end{align*}
Let $\u_1=X\alpha_1$ and $\u_2=X\alpha_2$. For any fixed $\alpha_1,\alpha_2\in\mathcal S_{n,s}$, with a probability $1-\delta$ we can have 
\begin{align*}
&\hspace*{-0.2in}\frac{1}{n}|(X\alpha_1)^{\top}(I - A^{\top}A)(X\alpha_2)|=O\left(\rho^+_s\sqrt{\frac{\log(1/\delta)}{m}}\right)
\end{align*}
where we use 
\[
\max_{\alpha\in \mathcal S_{n,s}}\frac{\|X\alpha\|_2^2}{n}\leq \max_{\alpha\in \mathcal K_{n,s}}\frac{\|X\alpha\|_2^2}{n} = \rho^+_s
\]
Then by using Lemma 3.3 in~\cite{DBLP:journals/corr/abs-1109-4299} about the entropy of $\mathcal S_{n,s}$ and the union bound, we can arrive at the following upper bound for $\sigma_{s}$.
\vspace*{-0.05in}
\begin{thm}
With a probability  $1-\delta$, we have 
\begin{align*}
\sigma_s\leq O\left(\rho^+_s\sqrt{\frac{(\log(1/\delta)+s\log(n/s))}{m}}\right)
\end{align*}
\end{thm}
{\bf Remark 9:} With above result, we can further understand the condition $\sigma_{16s}\leq \rho^-_{16s}$, which amounts to
$$O\left(\rho^+_{16s}\sqrt{\frac{(\log(1/\delta)+s\log(n/s))}{m}}\right)\leq \rho^-_{16s},$$ i.e., $m\geq \Omega(\kappa^2_{16s}(\log(1/\delta)+s\log(n/s)))$ where $\kappa_{16s}=\rho^+_{16s}/\rho^-_{16s}$ is the restricted condition number of the data matrix.

\section{Numerical Experiments}\label{sec:empirical}
\vspace*{-0.05in}
In this section, we provide a case study in support of DSRR and  the theoretical analysis, and a demonstration of the application of DSRR to distributed optimization.  

\textbf{A case study on text classification.}
We use the RCV1-binary data~\cite{Lewis:2004:RNB:1005332.1005345} to conduct a case study. The data contains $ 697,641$ documents and $47,236$ features. We use a splitting $677,399/20,242$ for training and testing. The feature vectors were normalized such that the $\ell_2$ norm is equal to 1.  We only report the results using random hashing since it is the most efficient, while other randomized reduction methods (except for random sampling) have similar performance. For the loss function, we use both the squared hinge loss (smooth) and the hinge loss (non-smooth). We aim to examine two questions related to our analysis and motivation (i) how does the value of $\tau$ affect the recovery error? (ii) how does the number of samples $m$ affect the recovery error? 

We vary the value of $\tau$ among $0, 0.1, 0.2, \ldots, 0.9$,  the value of $m$ among $1024, 2048, 4096, 8192$, and the value of $\lambda$ among $0.001, 0.00001$. Note that $\tau=0$ corresponds to the randomized reduction approach without the sparse regularizer. The results averaged over 5 random trials are shown in Figure~\ref{fig:1} for the squared hinge loss and in Figure~\ref{fig:1-1} for the hinge loss. We first analyze the results in Figure~\ref{fig:1}. We can observe that when $\tau$ increases the ratio of $\frac{\|[\at_*]_{\Se^c}\|_1}{\|[\at_*]_\Se - [\alpha_*]_\Se\|_1}$  decreases  indicating that the magnitude of  dual variables for the original  non-support vectors decreases. This is intuitive and consistent with our motivation. The recovery error of the dual solution (middle) first decreases and then increases. This can be partially explained by  the theoretical result in~Theorem~\ref{thm:1}. When the value of $\tau$ becomes larger than a certain threshold  making  $\tau>\|\Delta\|_\infty$ hold, then Theorem~\ref{thm:1} implies that a larger $\tau$ will lead to a larger error. On the other hand, when $\tau$ is less than the threshold, the dual recovery error will decrease as $\tau$ increases. 
In addition, the figures exhibit that the thresholds for larger $m$ are smaller which is consistent with our analysis of $\|\Delta\|_\infty=O(\sqrt{1/m})$. The difference between $\lambda=0.001$ and $\lambda=0.00001$ is because that smaller  $\lambda$ will lead to larger $\|\w_*\|_2$. In terms of the hinge loss, we observe similar trends, however, the recovery is much more difficult  than that for squared hinge loss especially when the value of $\lambda$ is small. 


\begin{figure}[t]
\centering
\includegraphics[scale=0.17]{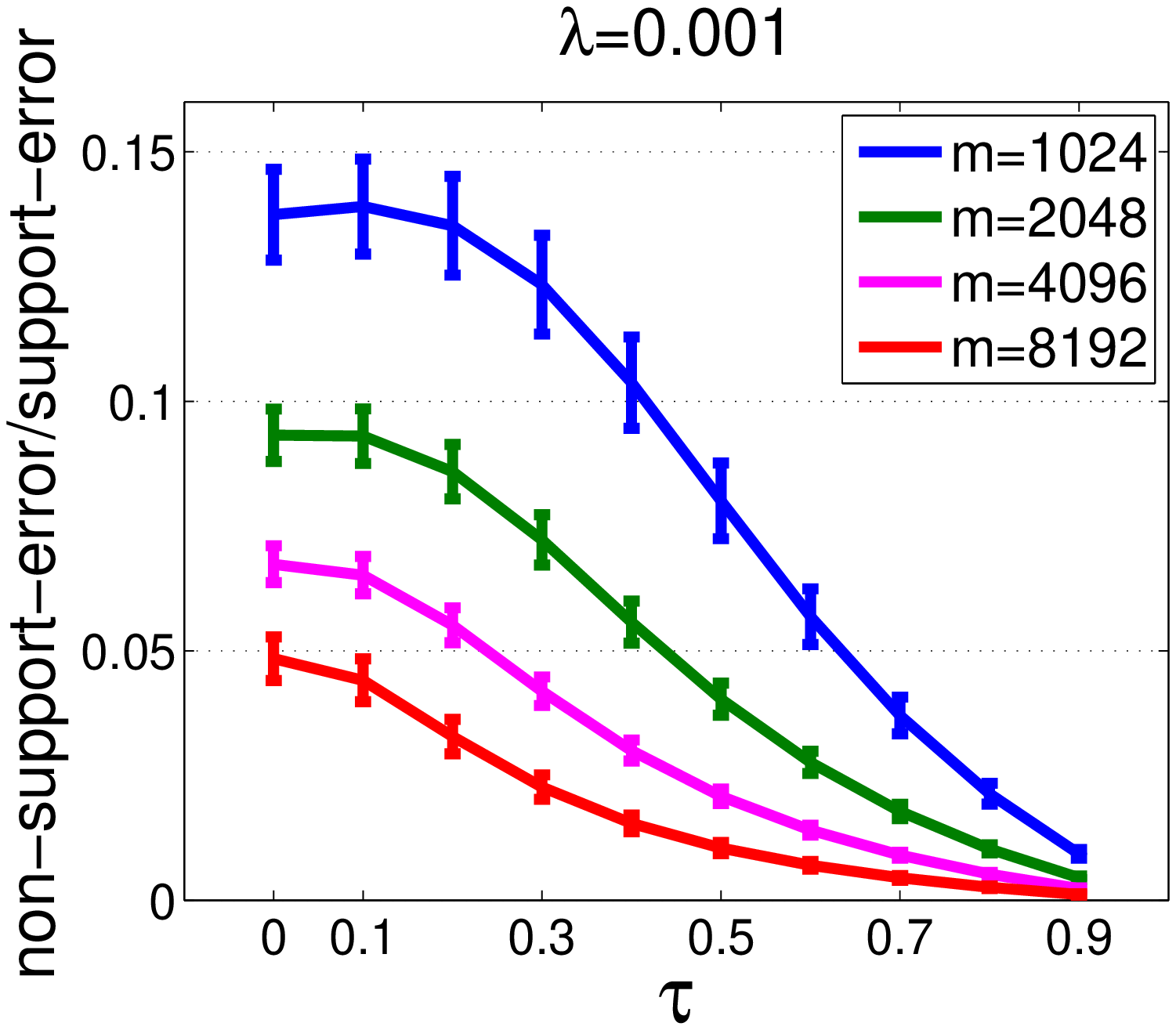}\hspace*{-0.1in}
\includegraphics[scale=0.17]{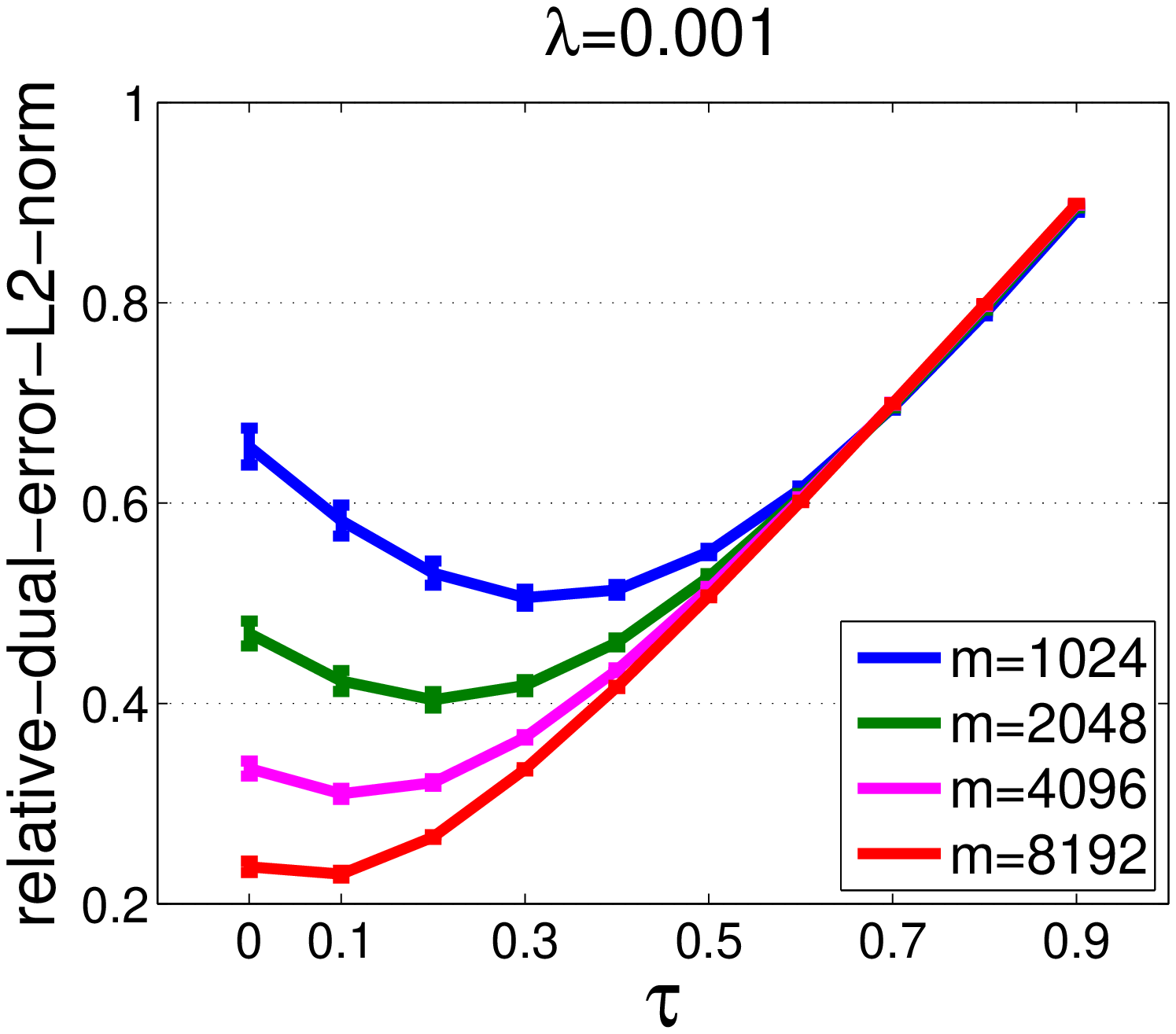}\hspace*{-0.1in}
\includegraphics[scale=0.17]{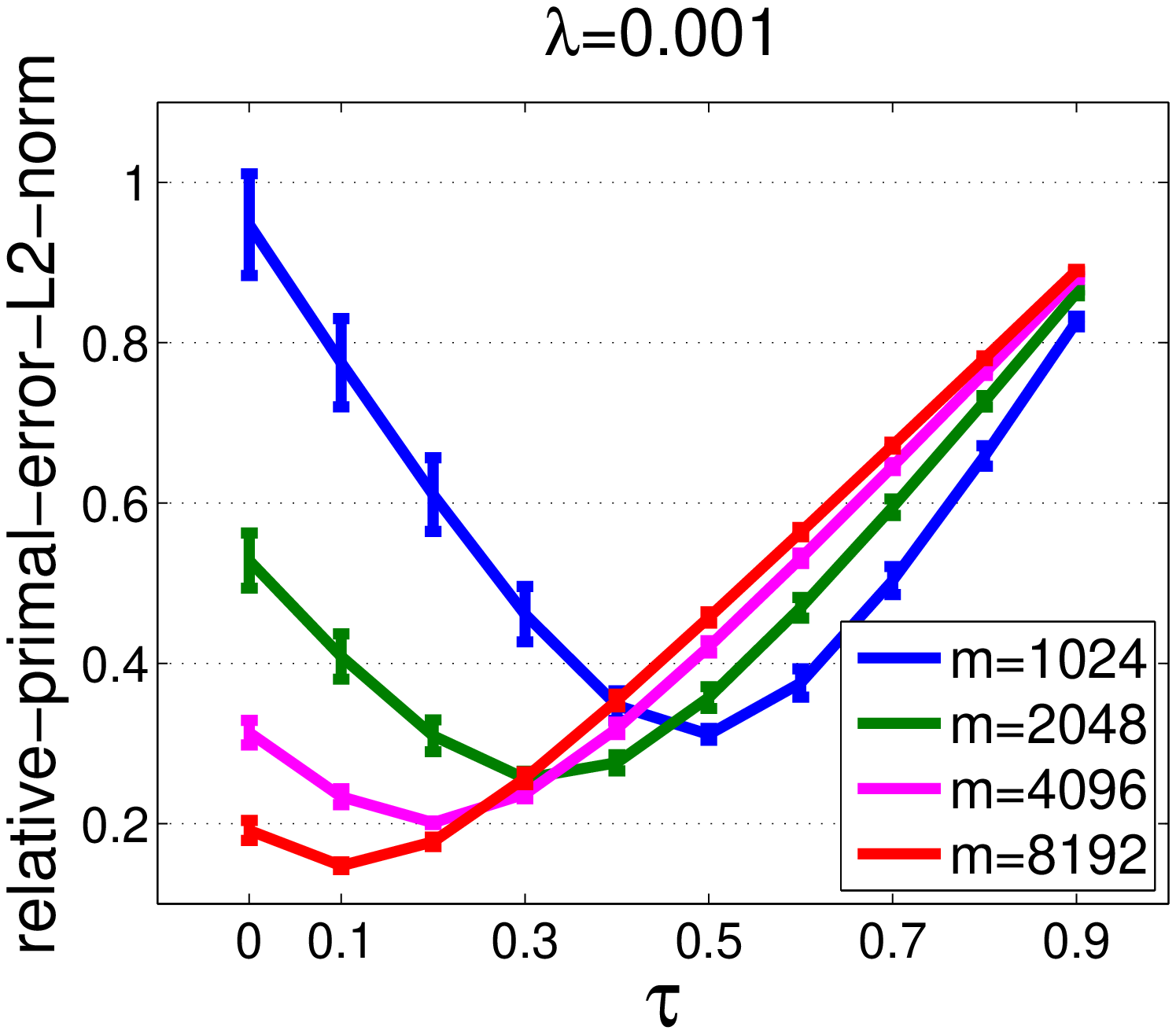}

\includegraphics[scale=0.178]{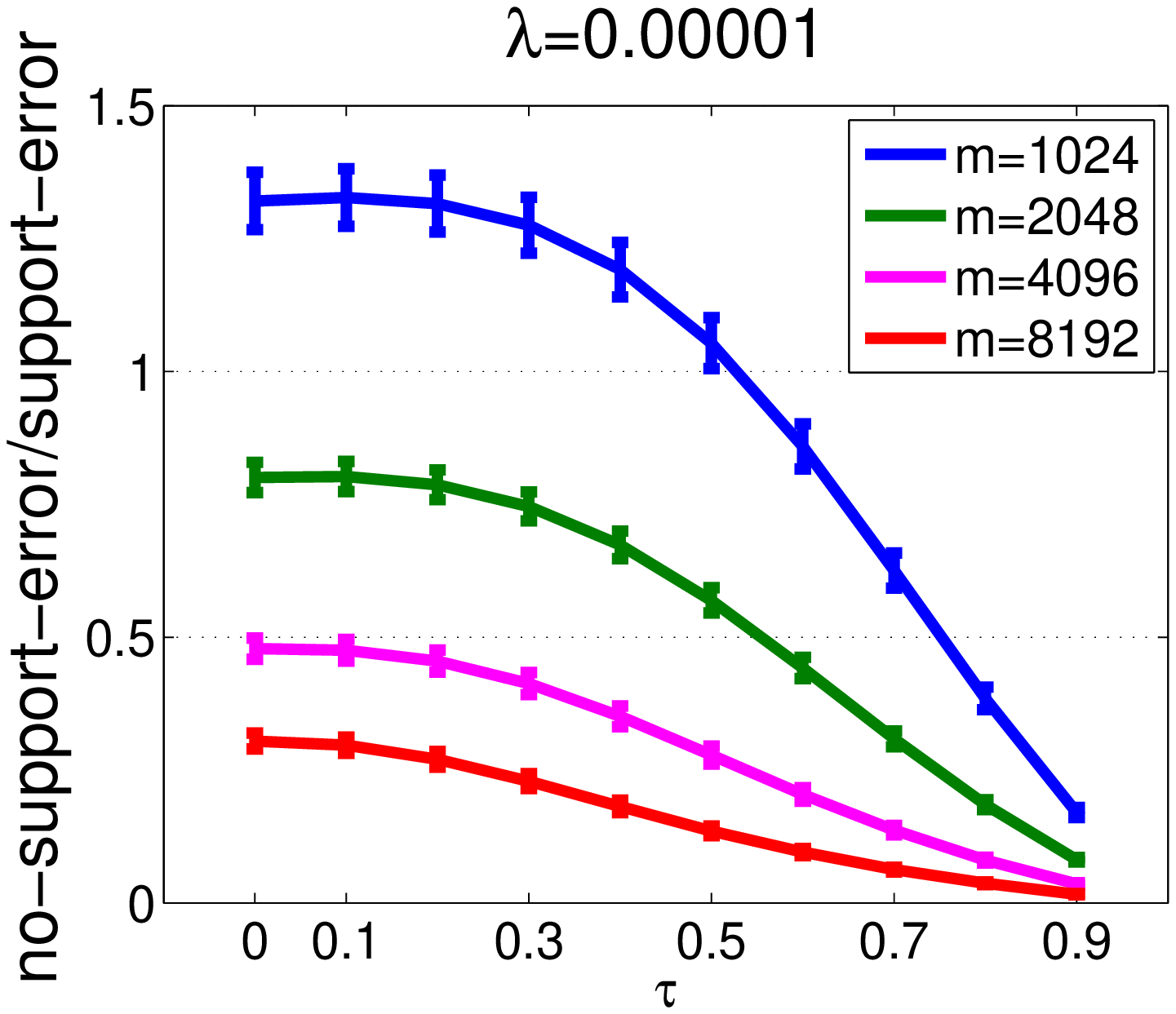}\hspace*{-0.1in}
\includegraphics[scale=0.178]{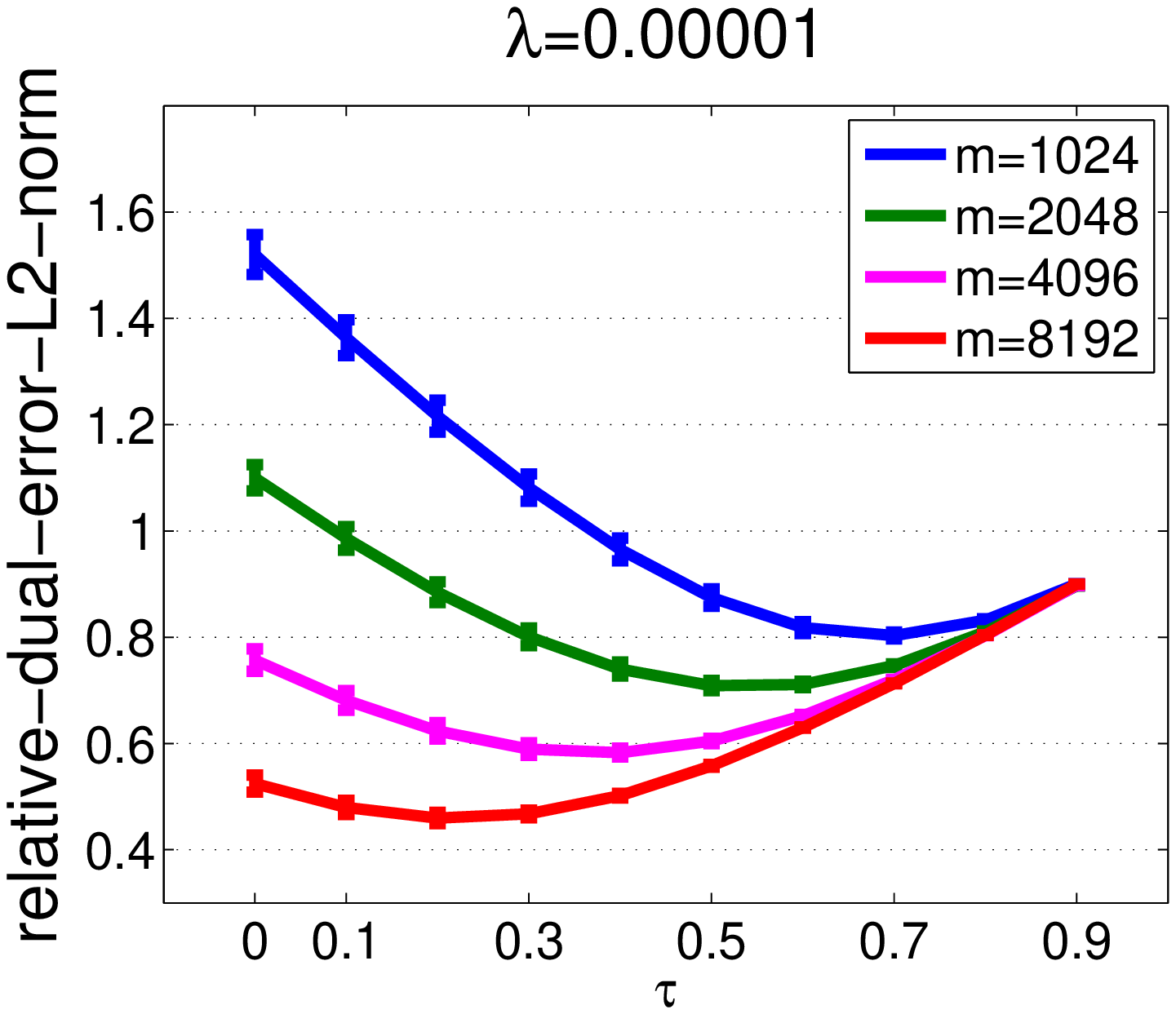}\hspace*{-0.1in}
\includegraphics[scale=0.178]{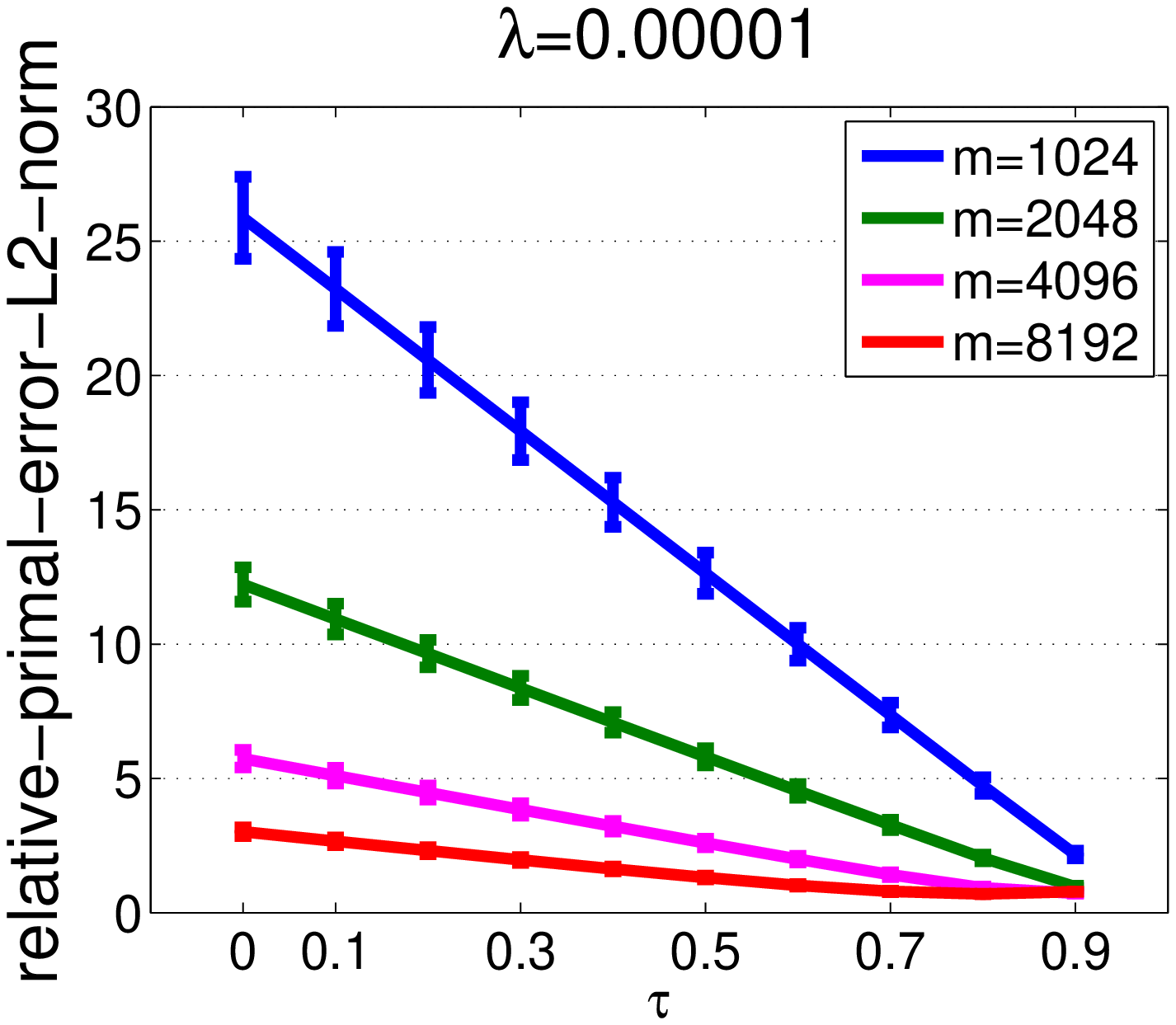}
\vspace*{-0.2in}
\caption{Recovery error for squared hinge loss. From left to right: $\frac{\|[\at_*]_{\Se^c}\|_1}{\|[\at_*]_\Se - [\alpha_*]_\Se\|_1}$ vs $\tau$, $\frac{\|\at_* - \alpha_*\|_2}{\|\alpha_*\|_2}$ vs $\tau$, and $\frac{\|\wt_* - \w_*\|_2}{\|\w_*\|_2}$ vs $\tau$. 
}\label{fig:1}
\vspace*{0.1in}
\includegraphics[scale=0.185]{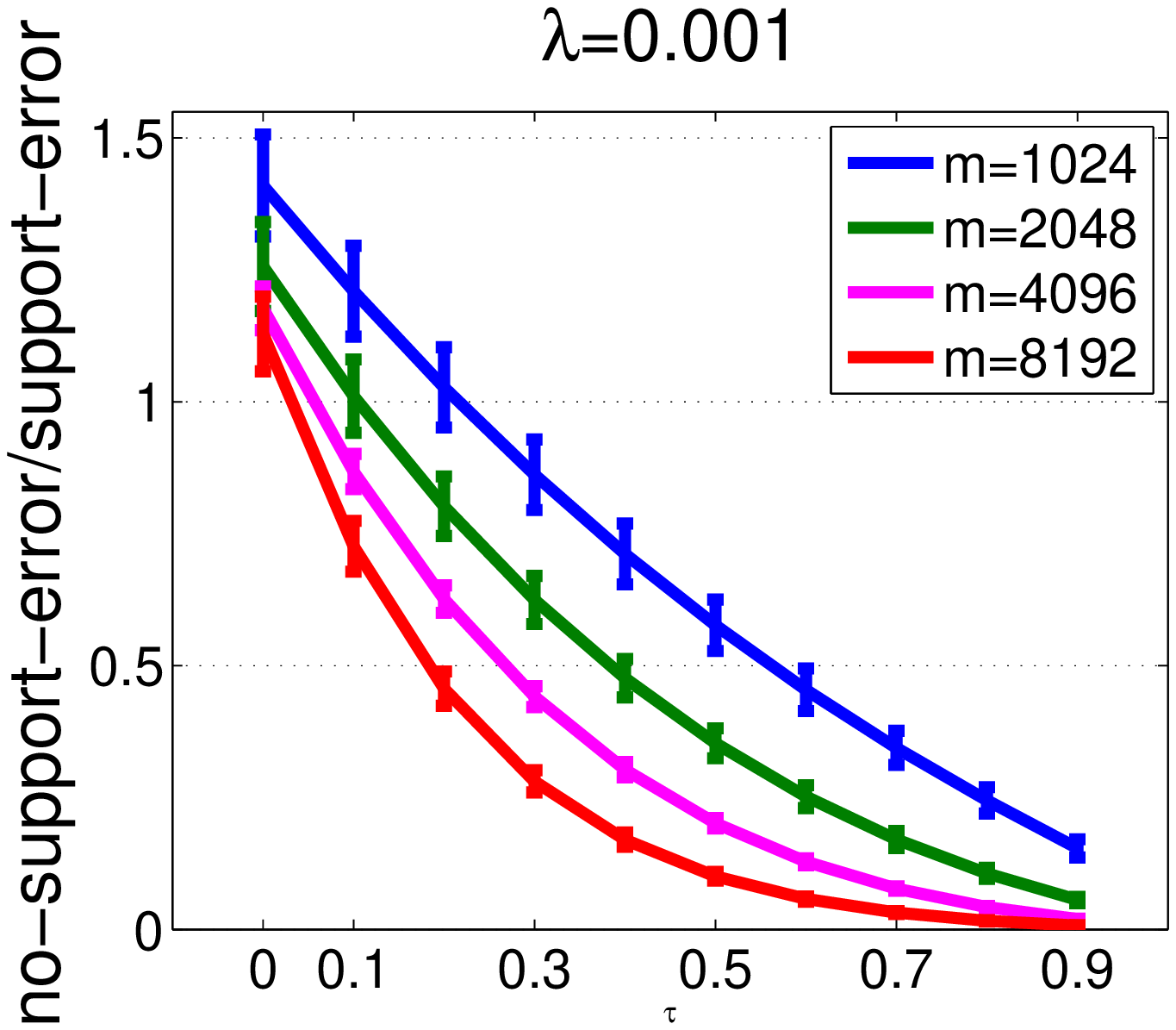}\hspace*{-0.1in}
\includegraphics[scale=0.185]{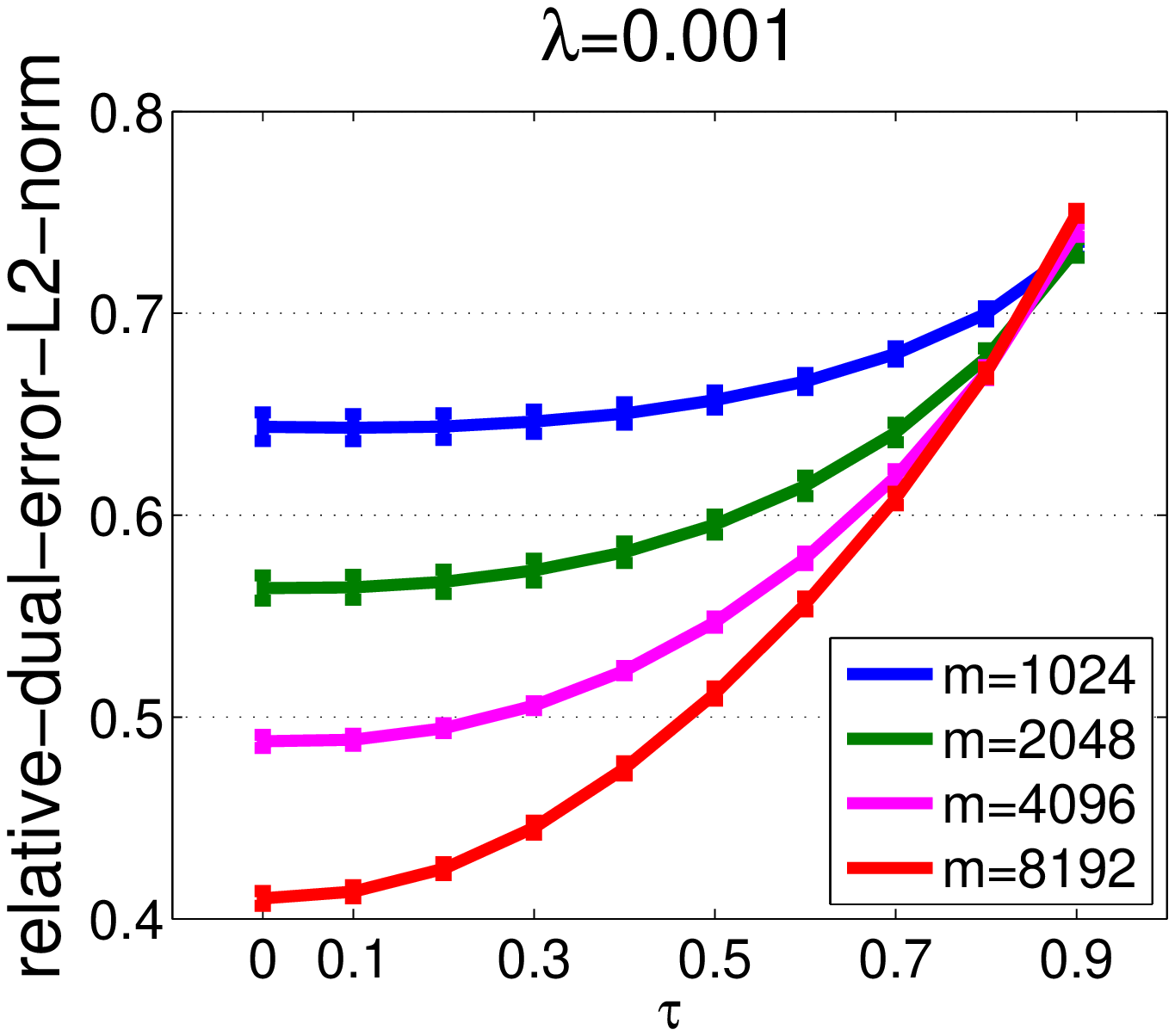}\hspace*{-0.1in}
\includegraphics[scale=0.185]{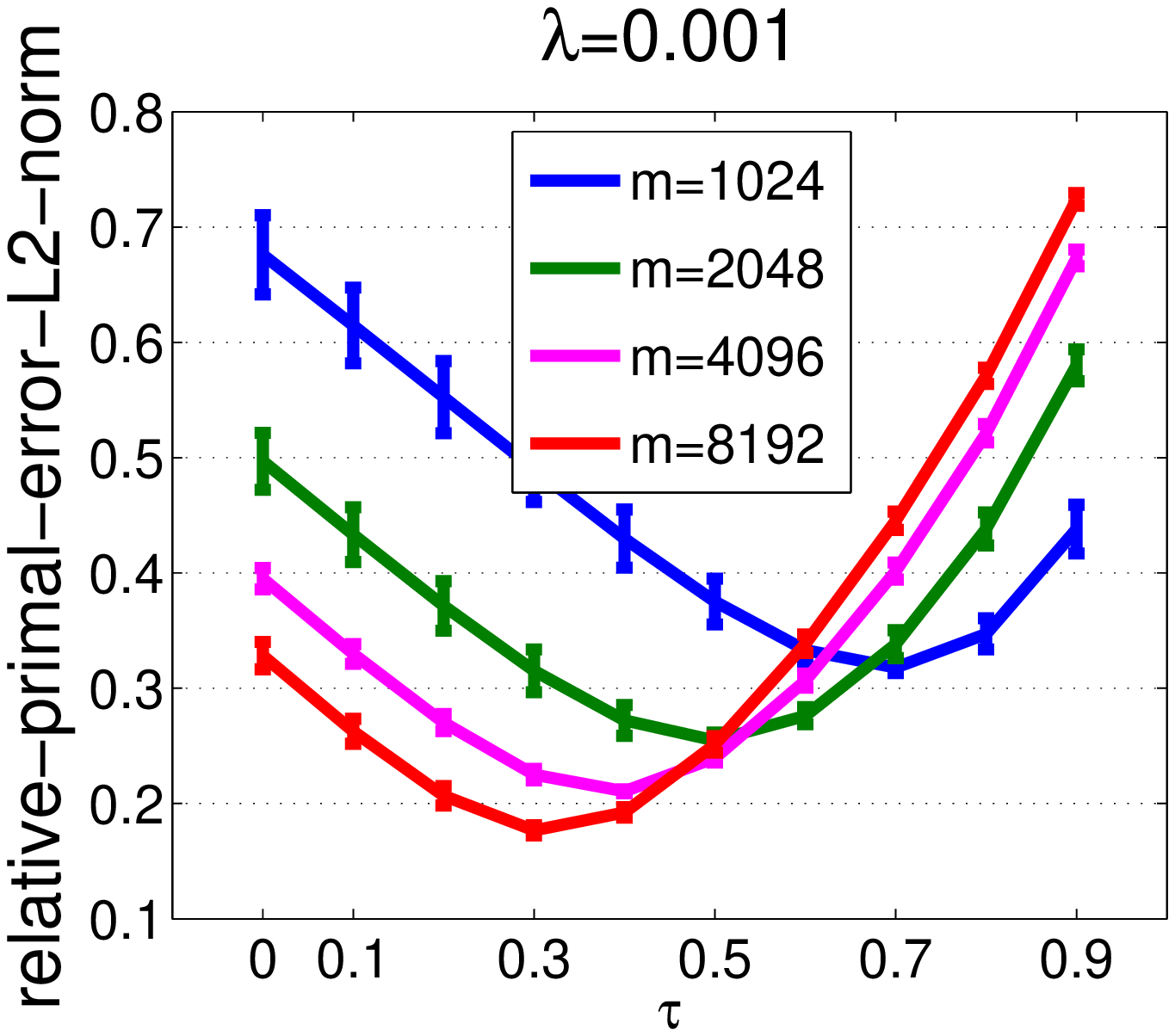}

\includegraphics[scale=0.185]{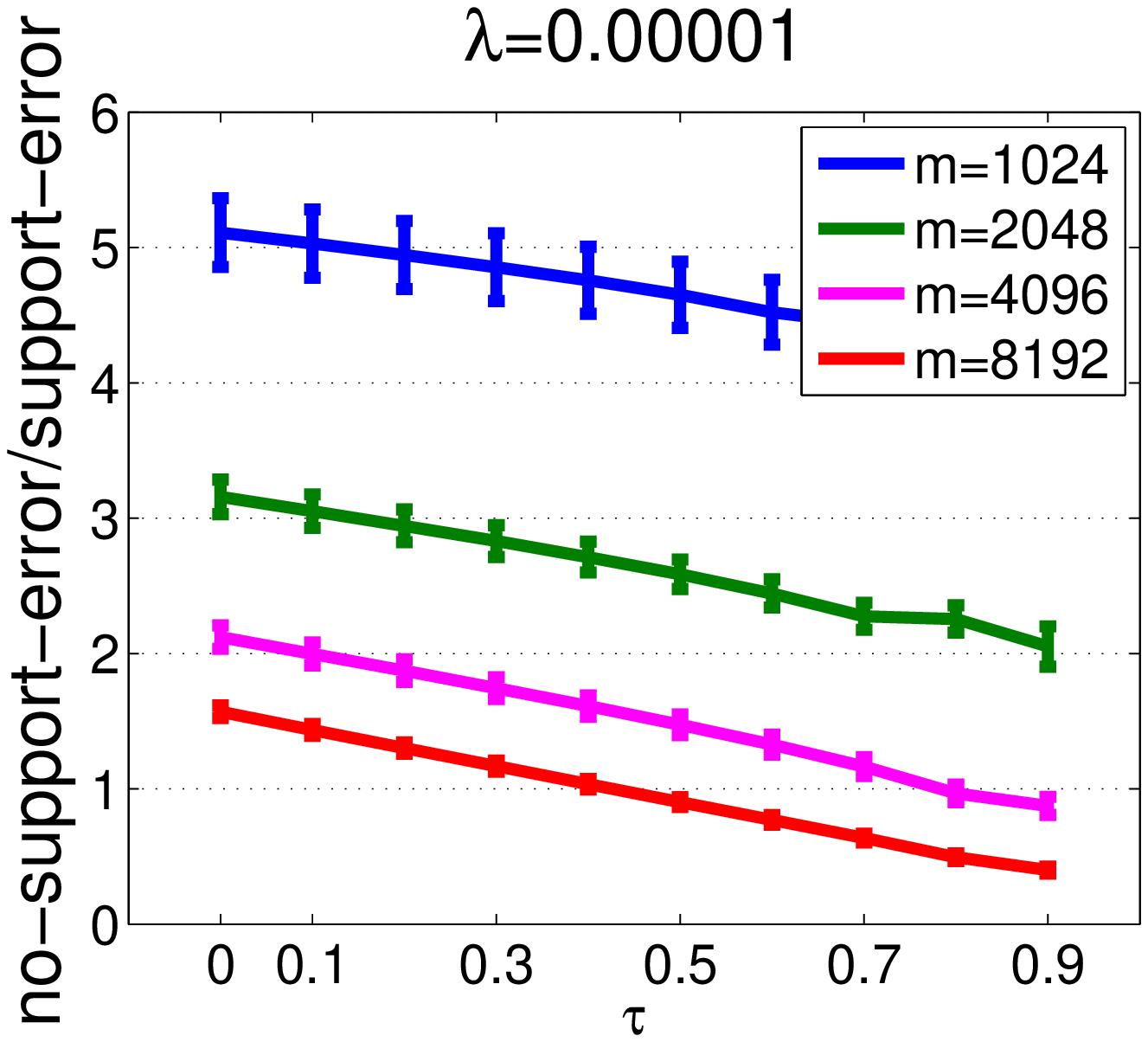}\hspace*{-0.1in}
\includegraphics[scale=0.185]{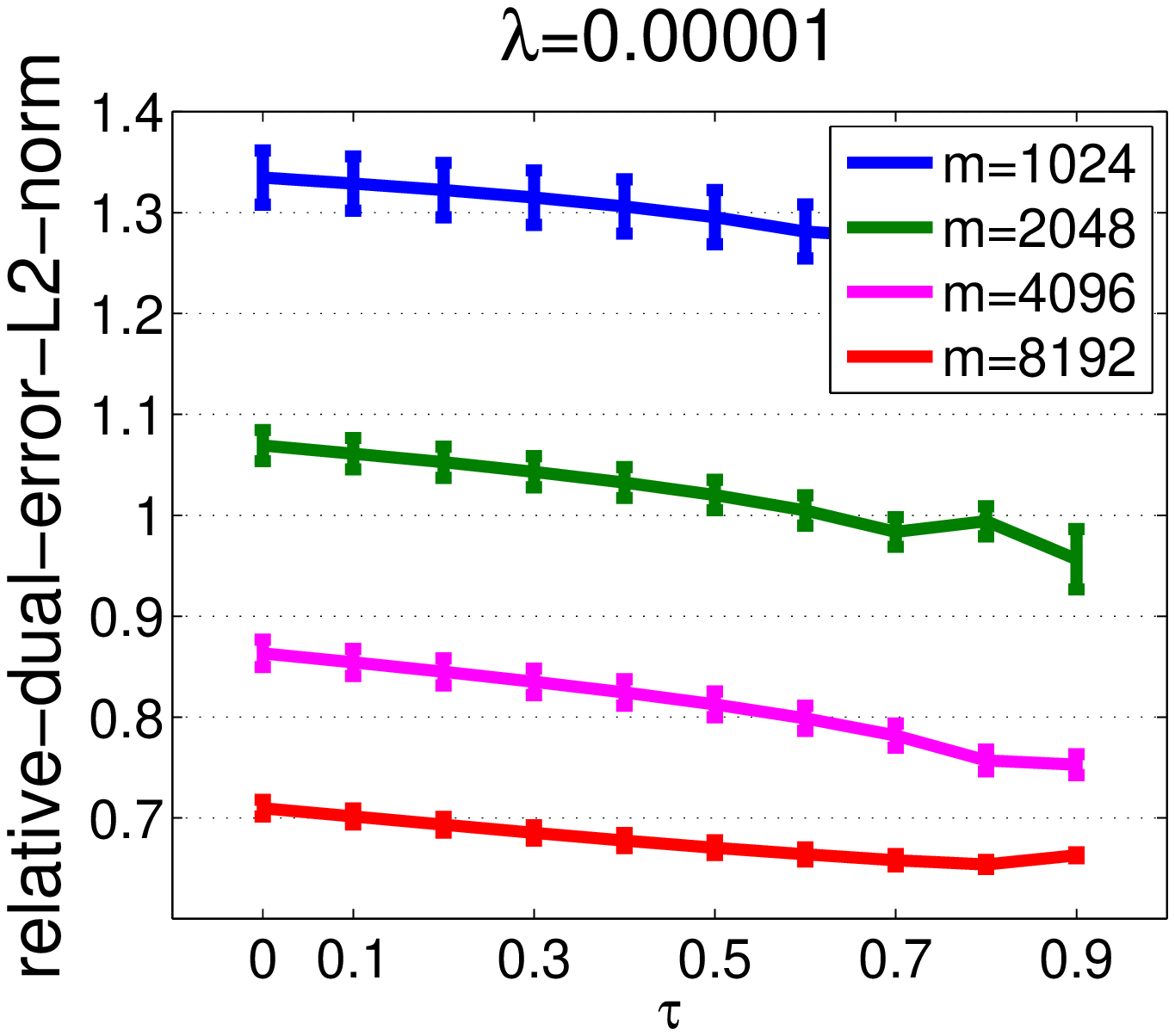}\hspace*{-0.1in}
\includegraphics[scale=0.185]{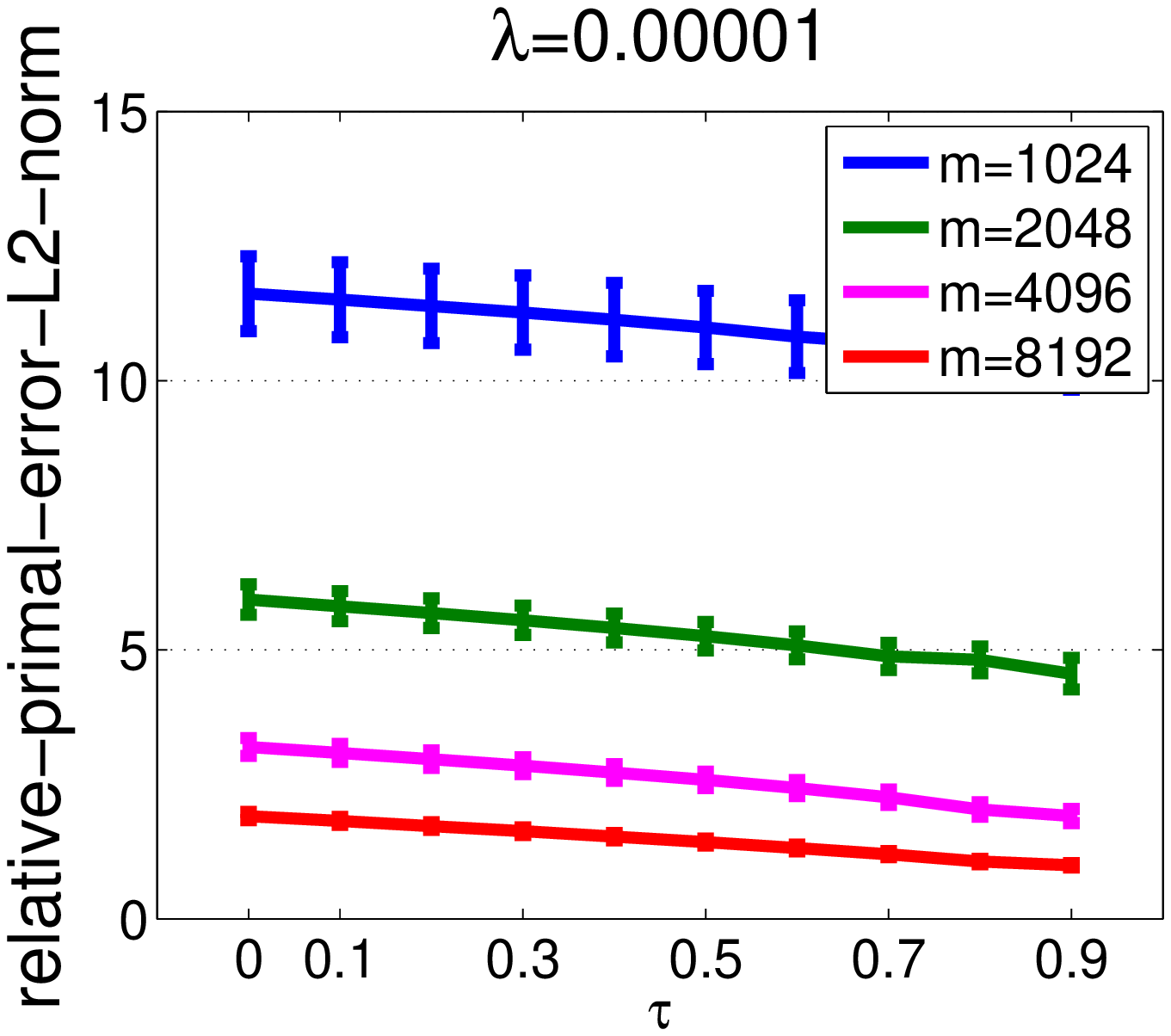}
\vspace*{-0.1in}
\caption{Same curves as above but for non-smooth hinge loss.}\label{fig:1-1}
\vspace*{-0.25in}
\end{figure}

\textbf{An application to distributed learning.} Although in some cases the solution learned in the reduced space  can provide sufficiently good performance, it usually performs worse than the optimal solution that solves the original problem  and sometimes the performance gap between them can not be ignored as seen in following experiments. To address this issue, we combine the benefits of distributed learning and the proposed randomized reduction methods for solving big data problems.  When data is too large and sits on multiple machines, distributed learning can be employed to solve the optimization problem.  In distributed learning, individual machines iteratively solve sub-problems associated with the subset of  data on them and communicate some global variables (e.g., the primal solution $\w\in\R^d$) among them. When the dimensionality $d$ is very large, the total communication cost could be very high. To reduce the total communication cost, we propose to first solve the reduced data problem and then use the found solution as the initial solution to the distributed learning  for the original data. 

Below, we demonstrate the effectiveness of DSRR for  the recently proposed distributed stochastic dual coordinate ascent (DisDCA) algorithm~\cite{DBLP:conf/nips/Yang13}. The procedure is (1) reduce original high-dimensional data to very low dimensional space on individual machines; (2) use DisDCA to solve the reduced problem; (3) use the optimal dual solution to the reduce problem as an initial solution to DisDCA for solving the original problem. We record the running time for randomized reduction in step 1 and optimization of the reduced problem in step 2, and the optimization of the original problem in step 3. We compare the performance of four methods (i) the {\bf DSRR} method that uses the model of the reduced problem solved by DisDCA to make predictions, (ii) the method that uses the recovered model in the original space, referred to as {\bf DSRR-Rec}; (iii) the method that uses the dual solution to the reduced problem as an initial solution of DisDCA and runs  it for the original problem with  $k=1$ or $2$ communications  (the number of updates before each communication is set to the number of examples in each machine), referred to as {\bf DSRR-DisDCA-$k$};  and (iv) the distributed  method that directly solves the original problem by {\bf DisDCA}.  For DisDCA to solve the original problem, we stop running  when its performance on the testing data does not improve.  Two data sets are used, namely RCV1-binary, KDD 2010 Cup data. 
 For KDD 2010 Cup data, we use the one available on LibSVM data website. 
The statistics  of the two data sets are summarized in Table~\ref{tab:data}. The results averaged over 5 trials are shown in Figure~\ref{fig:2}, which  exhibit that the performance of DSRR-DisDCA-1/2  is remarkable in the sense that it achieves almost the same performance of directly training on the original data (DisDCA) and uses much less training time. In addition, DSRR-DisDCA performs much better than DSRR and has small computational overhead. 

\begin{table}[t]
\caption{Statistics of datasets}
\begin{center}
\small{
\begin{tabular}{lllll}
\toprule
Name & \#Training&\#Testing &\#Features&\#Nodes \\
RCV1& 677,399& 20,242& 47, 236 & 5 \\
KDD  & 8,407,752 &748,401&29,890,095& 10\\
\bottomrule
\end{tabular}
}
\end{center}
\label{tab:data}
\vspace*{-0.2in}
\end{table}

\begin{figure}[t]
\centering
\includegraphics[scale=0.2]{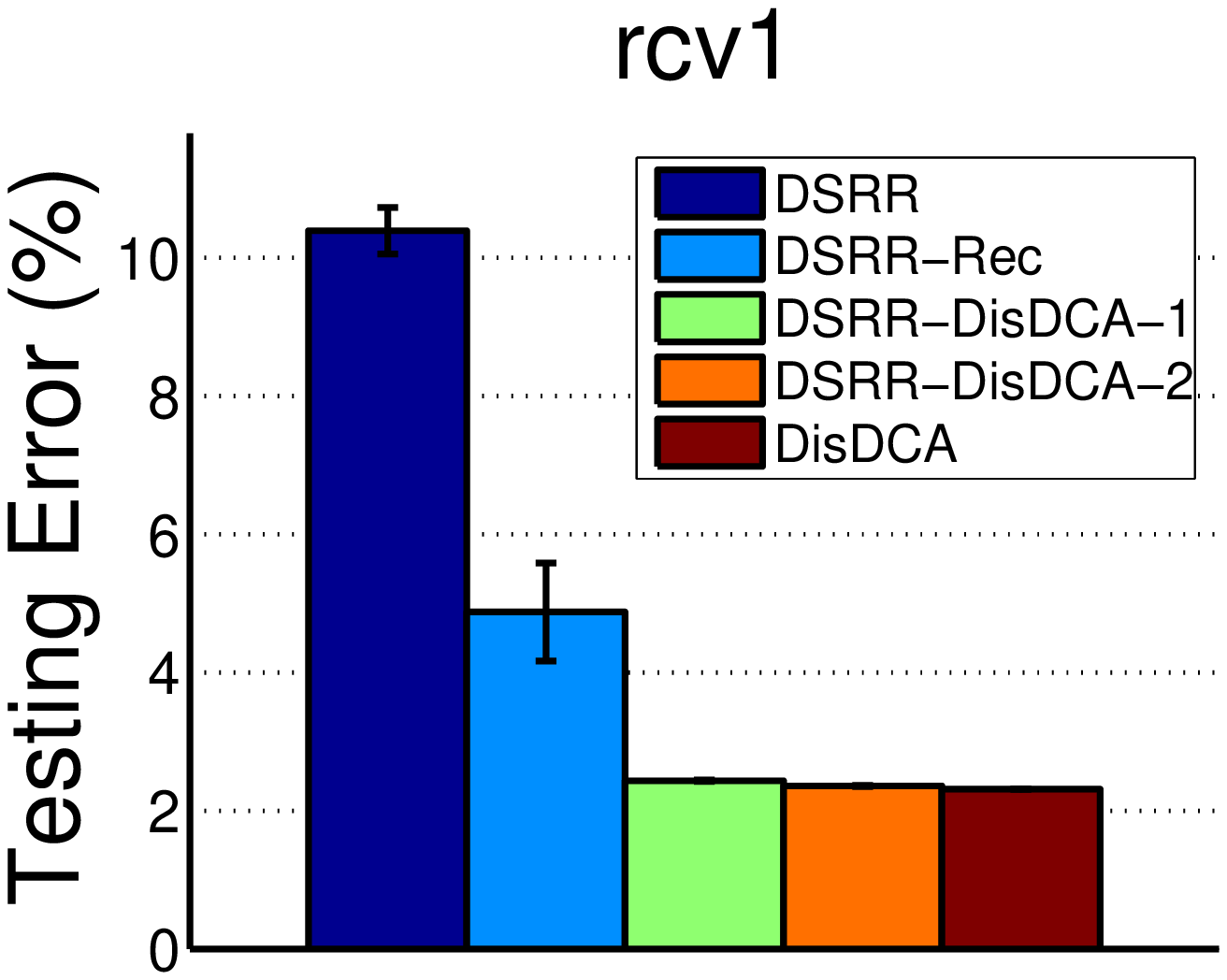}\hspace*{0.3in}
\includegraphics[scale=0.2]{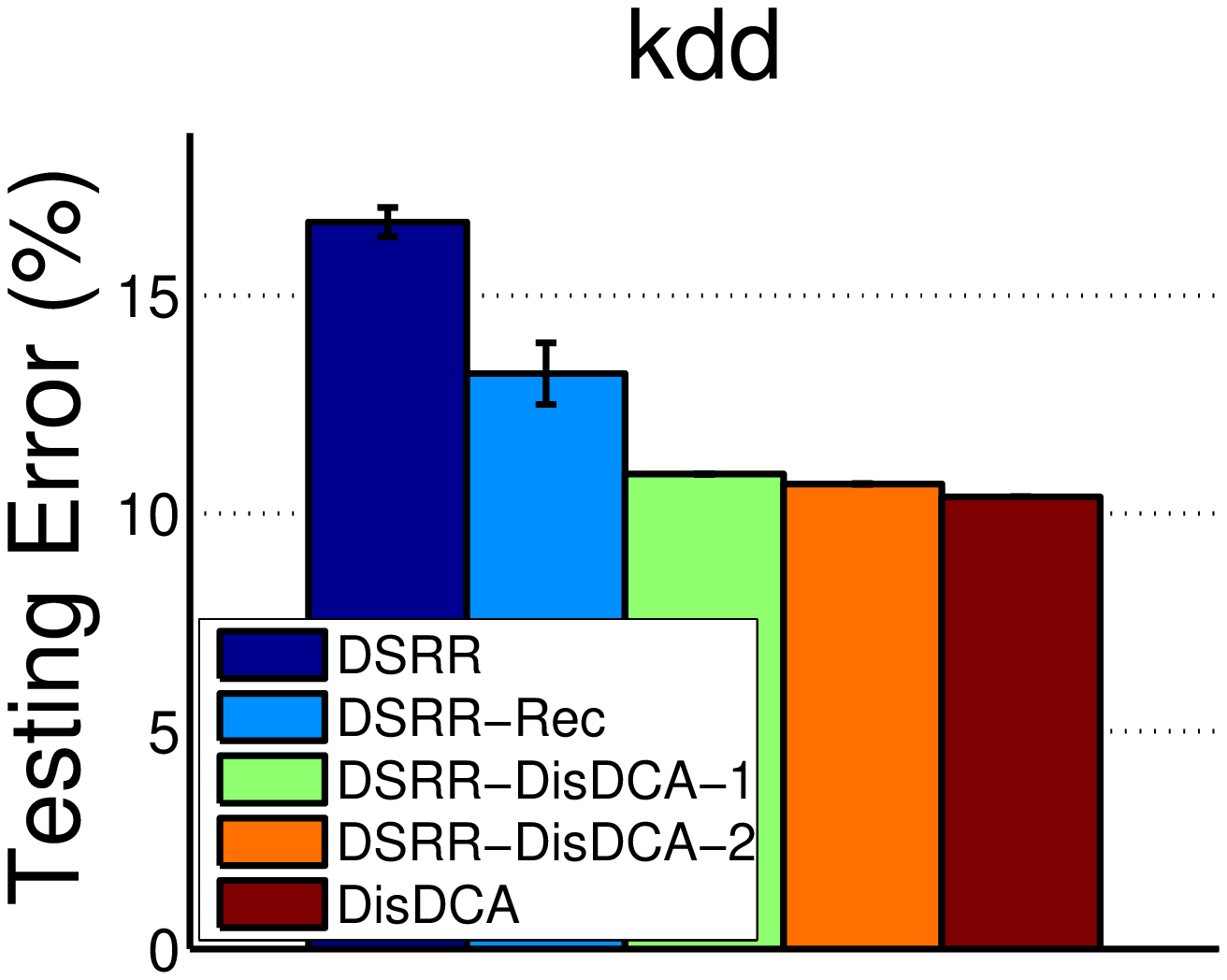}\hspace*{-0.1in}

\includegraphics[scale=0.2]{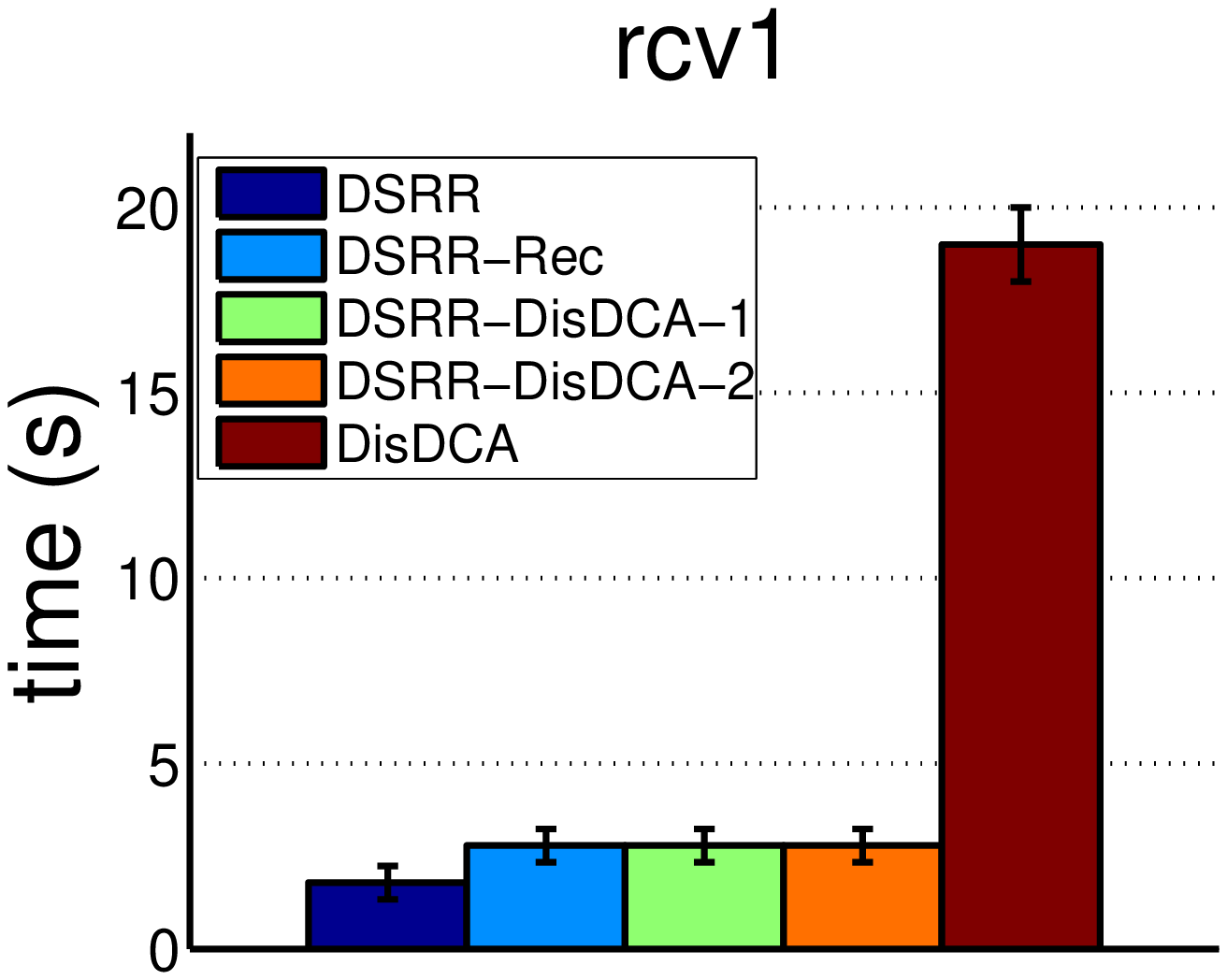}\hspace*{0.3in}
\includegraphics[scale=0.2]{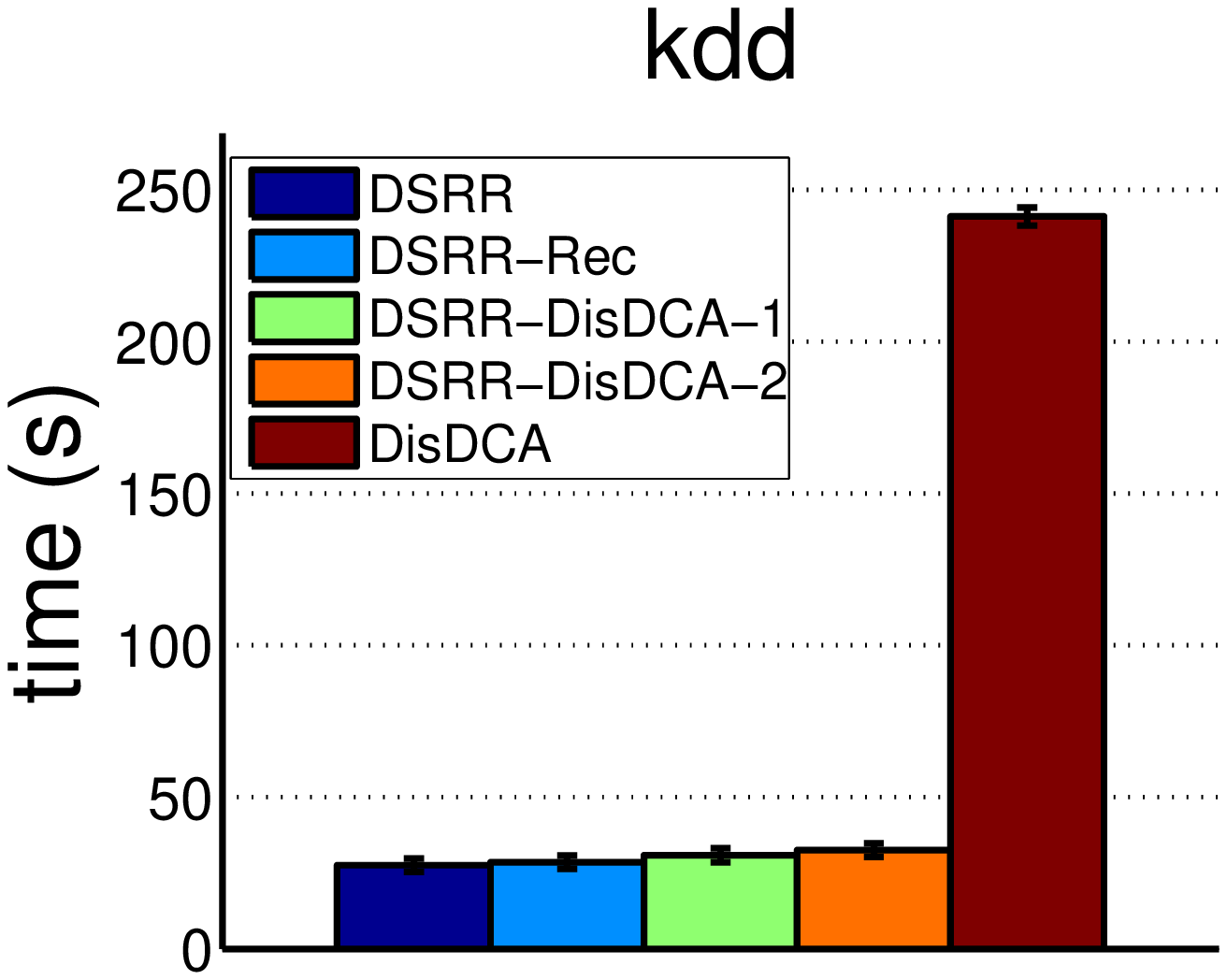}\hspace*{-0.1in}
\caption{Top: Testing error for different methods.  Bottom: Training time for different methods. The value of $\lambda=10^{-5}$ and the value of $\tau=0.9$. The high-dimensional features are reduced to $m=1024$-dimensional space using random hashing. The loss function is the squared hinge loss. }\label{fig:2}
\end{figure}

\section{Conclusions}
In this paper, we have proposed dual-sparse regularized randomized reduction methods for classification. We presented rigorous   theoretical analysis of the proposed dual-sparse randomized reduction methods in terms of recovery error under a mild condition that the optimal dual variable is (nearly) sparse for both smooth and non-smooth loss functions, and for various randomized reduction approaches. The numerical experiments validate our theoretical analysis and also demonstrate that the proposed reduction and recovery framework can benefit  distributed optimization  by providing a good initial solution.

\section*{Acknowledgements}
The authors would like to thank the anonymous reviewers for their helpful and insightful
comments. T. Yang was supported in part by NSF (IIS-1463988). R. Jin was partially supported by NSF IIS-1251031 and ONR N000141410631. 

\bibliography{all}
\bibliographystyle{icml2015}

\appendix
\section{Proof  of Theorem 1}\label{sec:A}
Let $\widehat F(\alpha)$ be defined as 
\[
\widehat F (\alpha) = \frac{1}{n}\sum_{i=1}^n\ell_i^*(\alpha_i) + \frac{1}{2\lambda n^2}\alpha^T\Xh^{\top}\Xh\alpha  + \frac{\tau}{n}\|\alpha\|_1
\]
Since $\at_* = \arg\min \widehat F(\alpha)$  therefore for any $g_*\in\partial \|\alpha_*\|_1$
\begin{align*}
0\geq& \hat F(\at_*) - \hat F(\alpha_*)\\
\geq& (\at_* - \alpha_*)^{\top}\left(\frac{1}{n}\nabla\ell^*(\alpha_*) + \frac{1}{\lambda n^2}\Xh^{\top}\Xh\alpha_*\right)\\
& +  \frac{\tau}{n}(\at_* - \alpha_*)^{\top}g_* + \frac{1}{2nL}\|\at_* - \alpha_*\|_2^2
\end{align*}
where we used the strong convexity of $\ell_i^*$ and its strong convexity modulus $1/L$.  
By the optimality condition of $\alpha_*$, we can have
\begin{align}\label{eqn:opt}
0&\geq (\alpha_* - \at_*)^{\top}\left(\frac{1}{n}\nabla\ell^*(\alpha_*) + \frac{1}{\lambda n^2}X^{\top}X\alpha_*\right) 
\end{align}
Combining the above two inequalities we have
\begin{align*}
0\geq& (\at_* - \alpha_*)^{\top}\frac{1}{n}\Delta+ \frac{\tau}{n}(\at_* - \alpha_*)^{\top}g_* + \frac{1}{2nL}\|\at_* - \alpha_*\|_2^2
\end{align*}
Since the above inequality holds for any $g_*\in\partial \|\alpha_*\|_1$, if we choose  $[g_*]_i= sign([\at_*]_i), i\in\Se^c$, then we have
\[
(\at_* - \alpha_*)^{\top}g_* \geq - \|[\at_*]_{\Se}-[\alpha_*]_{\Se}\|_1 + \|[\at_*]_{\Se^c}\|_1
\]
Combining the above inequalities leads to
\begin{align}\label{eqn:true}
(\tau+\|\Delta\|_\infty) \|[\at_*]_{\Se}-[\alpha_*]_{\Se}\|_1\geq &(\tau - \|\Delta\|_\infty)\|[\at_*]_{\Se^c}\|_1\notag\\
& + \frac{1}{2L}\|\at_* - \alpha_*\|_2^2
\end{align}
Assuming $\tau\geq 2\|\Delta\|_\infty$, we have
\begin{equation}\label{eqn:keyi}
\begin{aligned}
&\|\at_* - \alpha_*\|_2^2\leq 3\tau L \|[\at_*]_{\Se}-[\alpha_*]_{\Se}\|_1\\
&\|[\at_*]_{\Se^c}\|_1\leq 3 \|[\at_*]_{\Se}-[\alpha_*]_{\Se}\|_1
\end{aligned}
\end{equation}
Therefore, 
\begin{align*}
\|[\at_*-\alpha_*]_\Se\|_1^2\leq s\|\at_* - \alpha_*\|_2^2&\leq 3\tau Ls \|[\at_*]_{\Se}-[\alpha_*]_{\Se}\|_1
\end{align*}
leading to the result
\begin{align*}
\|[\at_*]_{\Se}-[\alpha_*]_{\Se}\|_1\leq 3\tau L s.  
\end{align*}
Combing this inequality with inequalities in~(\ref{eqn:keyi}) we have
\begin{align*}
 &\|[\at_*]_{\Se^c}\|_1\leq 9\tau L s,\quad \|\at_*-\alpha_*\|_2\leq 3\tau L\sqrt{s}.
\end{align*}

\section{Proof of Theorem 3}\label{sec:B}
Following the same proof of Theorem 1, we first notice that inequality~(\ref{eqn:true}) holds for $L=\infty$, i.e., 
\begin{align*}
(\tau+\|\Delta\|_\infty) \|[\at_*]_{\Se}-[\alpha_*]_{\Se}\|_1\geq &(\tau - \|\Delta\|_\infty)\|[\at_*]_{\Se^c}\|_1
\end{align*}
Therefore if $\tau\geq 2\|\Delta\|_\infty$, we have
\[
\|[\at_*]_{\Se^c}\|_1\leq 3\|[\at_*]_{\Se}-[\alpha_*]_{\Se}\|_1
\]
As a result, 
\begin{align*}
\frac{\|\at_* - \alpha_*\|_1}{\|\at_* - \alpha_*\|_2}&\leq \frac{\|[\at_*]_{\Se}-[\alpha_*]_{\Se}\|_1 + \|[\at_*]_{\Se^c}\|_1}{\|\at_* - \alpha_*\|_2}\\
&\leq \frac{4\|[\at_*]_{\Se}-[\alpha_*]_{\Se}\|_1}{\|\at_* - \alpha_*\|_2}\leq 4\sqrt{s}
\end{align*}
By the definition of $\mathcal K_{n,s}$, we have  $\displaystyle \frac{\at_* - \alpha_*}{\|\at_* - \alpha_*\|_2}\in \mathcal K_{n, 16s}$.  To proceed the proof, there exists $\widetilde g_*\in\partial |\at_*|_1$ such that 
\begin{align*}
0\geq&(\at_* - \alpha_*)^{\top}\left(\frac{1}{n}\nabla\ell^*(\at_*) + \frac{1}{\lambda n^2}\Xh^{\top}\Xh\at_*\right)\\
 &  + \frac{\tau}{n}(\at_* - \alpha_*)^{\top}\widetilde g_*
\end{align*}
Adding the above inequality with ~(\ref{eqn:opt}), we have
\begin{align*}
0&\geq  (\alpha_* - \at_*)^{\top}\left(\frac{1}{n}\nabla\ell^*(\alpha_*)   -\frac{1}{n}\nabla\ell^*(\at_*)\right)\\
& + (\alpha_* - \at_*)^{\top}\left(\frac{1}{\lambda n^2}X^{\top}X\alpha_* - \frac{1}{\lambda n^2}\Xh^{\top}\Xh\at_*\right) \\
& +\frac{\tau}{n}\left\|[\at_*]_{\Se^c}\right\|_1 - \frac{\tau}{n}\left\|[\at_*]_{\Se} - [\alpha_*]_{\Se}\right\|_1
\end{align*}
By convexity of $\ell^*$ we have
\[
(\alpha_* - \at_*)^{\top}\left[\frac{1}{n}\nabla\ell^*(\alpha_*)   -\frac{1}{n}\nabla\ell^*(\at_*)\right]\geq 0
\]
Thus, we have
\begin{align*}
\tau&\left\|[\at_*]_{\Se} - [\alpha_*]_{\Se}\right\|_1\geq \tau \left\|[\at_*]_{\Se^c}\right\|_1   \\
&+ (\alpha_* - \at_*)^{\top}\left(\frac{1}{\lambda n}X^{\top}X - \frac{1}{\lambda n}\Xh^{\top}\Xh\right)\alpha_* \\
& -  (\alpha_* - \at_*)^{\top}\left(\frac{1}{\lambda n}X^{\top}X - \frac{1}{\lambda n}\Xh^{\top}\Xh\right)(\alpha_* - \at_*)\\
& + \frac{1}{\lambda n}(\alpha_* - \at_*)^{\top}X^{\top}X(\alpha_* - \at_*)
\end{align*}
Since
\begin{align*}
 (\alpha_* - \at_*)^{\top}\Delta \geq -\|\Delta\|_\infty\|\alpha_* - \at_*\|_1,
\end{align*}
and $\tau\geq2\|\Delta\|_\infty$ and by the definition of $\rho^-_{s}, \sigma_{s}$, we have
\begin{align*}
\frac{3\tau}{2}&\left\|[\at_* - \alpha_*]_{\Se}\right\|_1\geq \frac{\tau}{2}\left\|[\at_*]_{\Se^c}\right\|_1  \\
&+ \frac{\rho^-_{16s} - \sigma_{16s}}{\lambda n}\|\at_* - \alpha_*\|_2^2
\end{align*}
Then the conclusion follows the same analysis as before.  

\section{Proof of Theorem 4}\label{sec:C}
Let $\hat F(\alpha)$ be defined as 
\[
\hat F (\alpha) = \frac{1}{n}\sum_{i=1}^n\ell_i^*(\alpha_i) + \frac{1}{2\lambda n^2}\alpha^T\Xh^{\top}\Xh\alpha  + \frac{\tau}{n}\|\alpha\|_1
\]
and $F(\alpha)$ be defined as
\[
F (\alpha) = \frac{1}{n}\sum_{i=1}^n\ell_i^*(\alpha_i) + \frac{1}{2\lambda n^2}\alpha^TX^{\top}X\alpha 
\]
Since $\at_* = \arg\min \hat F(\alpha)$  therefore for any $g_*\in\partial \|\alpha^s_*\|_1$
\begin{align*}
0\geq& \hat F(\at_*) - \hat F(\alpha^s_*)\\
\geq& (\at_* - \alpha^s_*)^{\top}\left(\frac{1}{n}\nabla\ell^*(\alpha^s_*) + \frac{1}{\lambda n^2}\Xh^{\top}\Xh\alpha^s_*\right)\\
& +  \frac{\tau}{n}(\at_* - \alpha^s_*)^{\top}g_* + \frac{1}{2nL}\|\at_* - \alpha^s_*\|_2^2\\
\end{align*}
where we used the strong convexity of $\ell_i^*$ and its strong convexity modulus $1/L$.  Due to the sub-optimality of $\alpha^s_*$, we have
\begin{align*}
\frac{1}{n}\|\alpha^s_* - \at_*\|_1\xi&\geq (\at_* - \alpha^s_*)^{\top}\left[\frac{1}{n}\nabla\ell^*(\alpha^s_*) + \frac{1}{\lambda n^2}X^{\top}X\alpha^s_*\right]
\end{align*}
Combining the above two inequalities we have
\begin{align*}
\frac{1}{n}\|\alpha^s_* - \at_*\|_1\xi\geq& (\at_* - \alpha^s_*)^{\top}\left(\frac{1}{\lambda n^2}(\Xh\Xh^{\top}-XX^{\top})\alpha^s_*\right) \\
&+ \frac{\tau}{n}(\at_* - \alpha^s_*)^{\top}g_* + \frac{1}{2nL}\|\at_* - \alpha^s_*\|_2^2
\end{align*}
Since the above inequality holds for any $g_*\in\partial \|\alpha^s_*\|_1$, if we choose  $[g_*]_i= sign([\at_*]_i), i\in\Se^c$, then we have
\begin{align}\label{eqn:true-2}
&(\xi + \tau)\|[\at_*]_{\Se}-[\alpha_*]_{\Se}\|_1\geq - \|\Delta\|_\infty\|\at_* - \alpha^s_*\|_1 \\
 &+ (\tau-\xi)\|[\at_*]_{\Se^c}\|_1 + \frac{1}{2L}\|\at_* - \alpha^s_*\|_2^2
\end{align}
Thus
\begin{align*}
&(\tau+\xi+\|\Delta\|_\infty) \|[\at_*]_{\Se}-[\alpha_*]_{\Se}\|_1\\
&\geq (\tau - \xi - \|\Delta\|_\infty)\|[\at_*]_{\Se^c}\|_1 + \frac{1}{2L}\|\at_* - \alpha^s_*\|_2^2
\end{align*}
Assuming $\tau\geq 2(\|\Delta\|_\infty+\xi)$, we have
\begin{equation}\label{eqn:keyi-2}
\begin{aligned}
&\|\at_* - \alpha^s_*\|_2^2\leq 3\tau L \|[\at_*]_{\Se}-[\alpha_*]_{\Se}\|_1\\
&\|[\at_*]_{\Se^c}\|_1\leq 3 \|[\at_*]_{\Se}-[\alpha_*]_{\Se}\|_1
\end{aligned}
\end{equation}
Therefore, 
\begin{align*}
\frac{\|[\at_*]_\Se - [\alpha_*]_\Se\|_1^2}{s}\leq \|\at_* - \alpha^s_*\|_2^2&\leq 3\tau L \|[\at_*]_{\Se}-[\alpha_*]_{\Se}\|_1
\end{align*}
leading to the result
\begin{align*}
&\|[\at_*]_{\Se}-[\alpha_*]_{\Se}\|_1\leq 3\tau L s
\end{align*}
Combing above inequality with inequalities in~(\ref{eqn:keyi}) we have
\begin{align*}
 &\|[\at_*]_{\Se^c}\|_1\leq 9\tau L s,\quad \|\at_*-\alpha^s_*\|_2\leq 3\tau L\sqrt{s}.
\end{align*}

\section{Proof of Theorem 7}\label{sec:D}
Recall the definition of  $\mathcal S_{n,s}$: 
\begin{align*}
\mathcal S_{n,s}=\{\alpha\in\R^n: \|\alpha\|_2\leq 1, \|\alpha\|_0\leq s\}
\end{align*}
Due to $conv(\mathcal S_{n,s})\subset\mathcal K_{n,s}\subset 2conv(\mathcal S_{n,s})$, 
for any $\alpha\in\mathcal K_{n,s}$, we can write it as $\alpha=2\sum_i\lambda_i\beta_i$ where $\beta_i\in\mathcal S_{n,s}$, $\sum_i\lambda_i=1$ and $\lambda_i\geq 0$, then we have 
\begin{align*}
&|(X\alpha)^{\top}(I - A^{\top}A)(X\alpha)|\\
&\leq 4\left|\left(X\sum_{i}\lambda_i\beta_i\right)^{\top}(I - A^{\top}A)\left(X\sum_i\lambda_i\beta_i\right)\right|\\
&\leq 4\sum_{ij}\lambda_i\lambda_j|(X\beta_i)^{\top}(I - A^{\top}A)(X\beta_j)|\\
&\leq 4\max_{\alpha_1, \alpha_2\in \mathcal S_{n,s}}|(X\alpha_1)^{\top}(I - A^{\top}A)(X\alpha_2)|\sum_{ij}\lambda_i\lambda_j\\
&= 4\max_{\alpha_1, \alpha_2\in \mathcal S_{n,s}}|(X\alpha_1)^{\top}(I - A^{\top}A)(X\alpha_2)
\end{align*}
Therefore 
\begin{align*}
&\max_{\alpha\in\mathcal K_{n,s}}|(X\alpha)^{\top}(I - A^{\top}A)(X\alpha)|\\
&\leq 4\max_{\alpha_1, \alpha_2\in \mathcal S_{n,s}}|(X\alpha_1)^{\top}(I - A^{\top}A)(X\alpha_2)
\end{align*}
Let $\u_1=X\alpha_1$ and $\u_2=X\alpha_2$. Following the Proof of Theorem 5, for any fixed $\alpha_1,\alpha_2\in\mathcal S_{n,s}$, with a probability $1-2\delta$ we  have 
\begin{align*}
&\frac{1}{n}|(X\alpha_1)^{\top}(I - A^{\top}A)(X\alpha_2)|\\
&\leq \frac{1}{n}\|X\alpha_1\|_2\|X\alpha_2\|_2\epsilon_{A,\delta}\leq \rho_s^+\epsilon_{A,\delta}\leq O\left(\rho^+_s\sqrt{\frac{\log(1/\delta)}{m}}\right)
\end{align*}
where we use 
\[
\max_{\alpha\in \mathcal S_{n,s}}\frac{\|X\alpha\|_2}{\sqrt{n}}\leq \max_{\alpha\in \mathcal K_{n,s}}\frac{\|X\alpha\|_2}{\sqrt{n}} = \sqrt{\rho^+_s}
\]
In order to extend the inequity to all $\alpha_1, \alpha_2\in\mathcal S_{n,s}$. We consider the $\epsilon$ proper-net of $\mathcal S_{n,s}$~\cite{DBLP:journals/corr/abs-1109-4299} denoted by $\mathcal S_{n,s}(\epsilon)$.  Lemma 3.3 in~\cite{DBLP:journals/corr/abs-1109-4299} shows that the entropy of $\mathcal S_{d,s}$, i.e., the cardinality of $\mathcal S_{n,s}(\epsilon)$ denoted $N(\mathcal S_{n,s},\epsilon)$ is bounded by
\[
\log N(\mathcal S_{n,s},\epsilon) \leq s\log\left(\frac{9n}{\epsilon s}\right)
\]
Then by using the union bound, we have with a probability $1-2\delta$, we have
\begin{align}\label{eqn:dis}
&\max_{\alpha_1\in\mathcal S_{n,s}(\epsilon)\atop \alpha_2\in\mathcal S_{n,s}(\epsilon)}\frac{1}{n}|(X\alpha_1)^{\top}(I - A^{\top}A)(X\alpha_2)|\nonumber\\
&\leq O\left(\rho^+_s\sqrt{\frac{\log(N^2(\mathcal S_{n,s},\epsilon)/\delta)}{m}}\right)\nonumber\\
&\leq O\left(\rho^+_s\sqrt{\frac{\log(1/\delta) + 2s\log(9n/\epsilon s)}{m}}\right)
\end{align}
To proceed the proof, we need the following lemma. 
\begin{lemma} \label{lemma:discrete-bound}
Let \begin{align*}
\Er_s(\alpha_2) = \max_{\alpha_1\in\Se_{n,s}} |\alpha_1^{\top}U\alpha_2|\\
\Er_s(\alpha_2,\epsilon) = \max_{\alpha_1\in\Se_{n,s}(\epsilon)} |\alpha_1^{\top}U\alpha_2|
\end{align*}
For  $\epsilon \in (0, 1/\sqrt{2})$, we have
\begin{align*}
\Er_s(\alpha_2)\leq \left(\frac{1}{1 - \sqrt{2}\epsilon}\right)\Er_s(\alpha_2,\epsilon)\end{align*}
\end{lemma}
\begin{proof} Let $U = \frac{1}{n}X^{\top}(I - A^{\top}A)X$.  Following  Lemma 9.2 of~\cite{koltchinskii2011oracle}, for any $\alpha, \alpha' \in \Se_{n,s}$, we can always find two vectors $\beta$, $\beta'$ such that
\[
\alpha-\alpha'=\beta-\beta', \ \|\beta\|_0 \leq s, \ \|\beta'\|_0 \leq s, \ \beta^\top \beta'=0.
\]
Let \begin{align*}
\Er_s(\alpha_2) = \max_{\alpha_1\in\Se_{n,s}} |\alpha_1^{\top}U\alpha_2|\\
\Er_s(\alpha_2,\epsilon) = \max_{\alpha_1\in\Se_{n,s}(\epsilon)} |\alpha_1^{\top}U\alpha_2|
\end{align*}
Thus
\[
\begin{split}
& |\langle \alpha - \alpha', U\alpha_2 \rangle| \leq  |\langle \beta, U\alpha_2 \rangle| + |\langle -\beta', U\alpha_2 \rangle| \\
= & \|\beta\|_2 \left|\left \langle \frac{\beta}{\|\beta\|_2}, U\alpha_2 \right \rangle\right| + \|\beta'\|_2 \left|\left \langle \frac{-\beta'}{\|\beta'\|_2}, U\alpha_2 \right \rangle\right| \\
\leq & (\|\beta\|_2+ \|\beta'\|_2) \Er_s(\alpha_2)  \leq \Er_s(\alpha_2)  \sqrt{2} \sqrt{\|\beta\|_2^2+ \|\beta'\|_2^2}\\
=& \Er_s(\alpha_2)  \sqrt{2} \|\beta-\beta'\|_2= \Er_s(\alpha_2)  \sqrt{2} \|\beta-\beta'\|_2\\
=& \Er_s(\alpha_2)  \sqrt{2} \|\alpha-\alpha'\|_2.
\end{split}
\]

Then, we have
\[
\begin{split}
& \Er_s(\alpha_2) = \max\limits_{\alpha \in \Se_{n,s}} |\alpha^{\top}U\alpha_2|\\
\leq & \max\limits_{\alpha \in \Se_{n,s}(\epsilon)} |\alpha^{\top}U\alpha_2| + \sup_{\alpha \in \Se_{n,s}\atop \alpha' \in \Se_{n,s}(\epsilon), \|\alpha-\alpha'\|_2 \leq \epsilon} \langle \alpha-\alpha', U\alpha_2 \rangle \\
\leq & \Er_s(\alpha_2, \epsilon) +  \sqrt{2} \epsilon \Er_s(\alpha_2)
\end{split}
\]
which implies
\[
\Er_s(\alpha_2) \leq \frac{\Er_s(\alpha_2, \epsilon)}{1-\sqrt{2} \epsilon}.
\]
\end{proof}

\begin{lemma} \label{lemma:discrete-bound-2}
Let \begin{align*}
\Er_s(\epsilon) = \max_{\alpha_2\in\Se_{n,s}}\Er_s(\alpha_2,\epsilon)=\max_{\alpha_1\in\Se_{n,s}\atop \alpha_2\in\Se_{n,s}(\epsilon)} |\alpha_1^{\top}U\alpha_2|\\
\Er_s(\epsilon,\epsilon)=\max_{\alpha_2\in\Se_{n,s}(\epsilon)}\Er_s(\alpha_2,\epsilon) = \max_{\alpha_1,\alpha_2\in\Se_{n,s}(\epsilon)} |\alpha_1^{\top}U\alpha_2|
\end{align*}
For  $\epsilon \in (0, 1/\sqrt{2})$, we have
\begin{align*}
\Er_s(\epsilon)\leq \left(\frac{1}{1 - \sqrt{2}\epsilon}\right)\Er_s(\epsilon,\epsilon)\end{align*}
\end{lemma}
The proof the above lemma follows the same analysis as that of Lemma 1. 
By combining Lemma 1 and Lemma 2, we have
\begin{align*}
&\sigma_s = \max_{\alpha_2\in\Se_{n,s}}\Er_s(\alpha_2)\leq \frac{\max_{\alpha_2\in\Se_{n,s}}\Er_s(\alpha_2, \epsilon)}{1-\sqrt{2} \epsilon}\\
&= \frac{1}{1-\sqrt{2}\epsilon}\Er_s(\epsilon)\leq \left(\frac{1}{1-\sqrt{2}\epsilon}\right)^2\Er_s(\epsilon,\epsilon)\\
& =  \left(\frac{1}{1-\sqrt{2}\epsilon}\right)^2 \max_{\alpha_1,\alpha_2\in\Se_{n,s}(\epsilon)} |\alpha_1^{\top}U\alpha_2|
\end{align*}
By combing the above inequality with inequality~(\ref{eqn:dis}), we have
\[
\sigma_s\leq \left(\frac{1}{1-\sqrt{2}\epsilon}\right)^2 O\left(\rho^+_s\sqrt{\frac{\log(1/\delta) + 2s\log(9n/\epsilon s)}{m}}\right)
\]

\end{document}